\documentclass{article}

\PassOptionsToPackage{numbers}{natbib}

\usepackage[final]{neurips_2021}

\usepackage[utf8]{inputenc} 
\usepackage[T1]{fontenc}    
\usepackage{hyperref}       
\usepackage{url}            
\usepackage{booktabs}       
\usepackage{amsfonts}       
\usepackage{nicefrac}       
\usepackage{microtype}      
\usepackage{xcolor}         
\usepackage{xfrac}

\usepackage{graphicx}
\usepackage{subfigure}
\usepackage{amsmath,amssymb,amsthm}
\usepackage{dsfont}
\usepackage{multirow}
\usepackage{todonotes,xspace,enumitem}
\usepackage{cleveref}
\usepackage{caption}
\usepackage{wrapfig}
\usepackage{appendix}
\usepackage{tcolorbox}

\usepackage{pifont}

\newtheorem{theorem}{Theorem}
\newtheorem{lemma}{Lemma}

\theoremstyle{definition}

\newtheorem{conjecture}{Conjecture}

\newcommand{\alphabar}{\bar{\alpha}}

\newcommand{\A}{\mathcal{A}}

\newcommand{\Ab}[1]{\A^{\textrm{bot}}_{#1}}
\newcommand{\At}[1]{\A^{\textrm{top}}_{#1}}
\newcommand{\R}{\mathbb{R}}
\newcommand{\set}[1]{\left\{#1\right\}}
\newcommand{\Y}{\mathcal{Y}}
\newcommand{\D}{\mathcal{D}}

\newcommand{\predm}[1]{\texttt{PredPower}_{M}(#1)}

\newcommand{\dtilde}{\widetilde{d}}
\newcommand{\utilde}{\widetilde{u}}
\newcommand{\defeq}{{:}{=}}

\newcommand{\Enu}[1]{\mathbb{E}_{(w,r,b)\sim \nu}\left[#1\right]}
\newcommand{\Enus}[1]{\mathbb{E}_{(w,r,b)\sim \nu^*}\left[#1\right]}
\newcommand{\Enut}[1]{\mathbb{E}_{(w,r,b)\sim \nut}\left[#1\right]}
\newcommand{\Eps}[1]{\mathbb{E}_{(x,y)\sim p^*}\left[#1\right]}
\newcommand{\Sd}{\mathbb{S}^{d+1}}
\newcommand{\Sddtilde}{\mathbb{S}^{d\dtilde+1}}
\newcommand{\uhat}{\widehat{u}}
\newcommand{\eps}{\eta}

\newcommand{\ED}[1]{\mathbb{E}_{\D}\left[#1\right]}

\newcommand{\ind}[1]{\mathds{1}\left[#1\right]}

\newcommand{\premise}{(\texttt{A})\xspace}

\newcommand{\abs}[1]{\left|#1\right|}
\newcommand{\norm}[1]{\left\|#1\right\|}
\newcommand{\iprod}[2]{\langle#1,#2\rangle}
\DeclareMathOperator*{\argmax}{arg\,max}
\DeclareMathOperator*{\argmin}{arg\,min}
\renewcommand{\L}{\mathcal{L}}
\newcommand{\Logit}{\mathrm{Logit}}

\newcommand{\cifar}{CIFAR-10}
\newcommand{\imnet}{ImageNet-10}

\newcommand{\mnist}{MNIST}

\newcommand{\evalname}{\texttt{DiffROAR}}
\newcommand{\semirealname}{\texttt{BlockMNIST}}
\newcommand{\semirealnameone}{\texttt{BlockMNIST-Top}}

\newcommand{\W}{\mathcal{W}}
\renewcommand{\P}{\mathcal{P}}
\newcommand{\nut}{\tilde{\nu}}

\title{Do Input Gradients Highlight\\Discriminative Features?}

\author{%
  Harshay Shah\thanks{{Part of the work completed after joining Google Research India}} \\
    Microsoft Research India \\
  \texttt{harshay@google.com} \\
	\And
	Prateek Jain\footnotemark[1] \\
	Microsoft Research India \\
	\texttt{prajain@google.com} \\
	\And
	Praneeth Netrapalli\footnotemark[1]  \\
	Microsoft Research India \\
	\texttt{pnetrapalli@google.com} \\
}

\begin{document}

\maketitle

\vspace{-8px}
\begin{abstract}
Post-hoc {gradient-based} interpretability methods~\citep{simonyan2013deep,smilkov2017smoothgrad} that provide instance-specific explanations of model predictions are often based on assumption \premise: \emph{~magnitude of input gradients---gradients of  logits with respect to input---{noisily} highlight discriminative task-relevant features}. In this work, we test the validity of assumption \premise~using a three-pronged approach:
\begin{enumerate}[leftmargin=*]
	\item	We develop an evaluation framework, \evalname, to test assumption \premise~on four image classification benchmarks. Our results {suggest} that (i) input gradients of standard models (i.e., trained on  original data) may grossly violate \premise, whereas (ii) input gradients of adversarially robust models satisfy \premise reasonably well. 
	\item We then introduce \semirealname, an  \texttt{MNIST}-based semi-real dataset, that \emph{by design} encodes {a priori} knowledge of   discriminative features. Our analysis on \semirealname~leverages this information to validate as well as characterize differences between input gradient attributions of standard and robust models.
	\item Finally, we theoretically prove that our empirical findings hold on a simplified version of the \semirealname~dataset. Specifically, we prove that input gradients of standard one-hidden-layer MLPs trained on this dataset do not highlight instance-specific ``signal'' coordinates, thus grossly violating \premise. 
\end{enumerate}
Our findings motivate the need to formalize and test common assumptions in interpretability in a falsifiable manner \cite{Leavitt2020TowardsFI}. We believe that the \evalname~framework and \semirealname~  datasets serve as sanity checks to audit interpretability methods; code and data available at \url{https://github.com/harshays/inputgradients}.
\end{abstract}

\vspace{-6px}
\section{Introduction}
\label{sec:intro}
\vspace{-4px}

Interpretability methods that provide instance-specific explanations of model predictions are often used to identify biased predictions~\cite{lime}, debug trained models~\cite{Adebayo2020DebuggingTF}, and aid decision-making in high-stakes domains such as medical diagnosis~\cite{zech2018variable,stiglic2020interpretability}.
A common approach for providing instance-specific explanations is \emph{feature attribution}. Feature attribution methods rank or score input coordinates, or features, in the order of their \emph{purported} importance in model prediction; coordinates achieving the top-most rank or score are considered most important for prediction, whereas those with the bottom-most rank or score are considered least important. 

\textbf{Input gradient attributions}. 
Ranking input coordinates based on the \emph{magnitude of input gradients} is a fundamental feature attribution technique~\cite{baehrens2010explain,simonyan2013deep} that undergirds well-known methods such as SmoothGrad~\cite{smilkov2017smoothgrad} and Integrated Gradients~\cite{sundararajan2017axiomatic}.	
Given instance $x$ and a trained model $\theta$ with prediction $\hat{y}$ on $x$, the input gradient attribution scheme 
(i) computes the input gradient $\nabla_x \Logit_{\theta}(x,\hat{y})$ of the logit
\footnote{In~\Cref{appendix:real}, we show that our results also hold for input gradients taken w.r.t. the \emph{loss}}
of the predicted label $\hat{y}$ and (ii) ranks the input coordinates in \emph{decreasing} order of their input gradient magnitude. 
Below we explicitly characterize the underlying intuitive assumption behind input gradient attribution methods:
\begin{center}
\begin{tcolorbox}[width=0.95\linewidth,boxsep=0pt,colback=white]
	\textbf{Assumption \premise}: \emph{Coordinates with larger input gradient magnitude are more relevant for model prediction compared to coordinates with smaller input gradient magnitude.} 
\end{tcolorbox}	
\end{center}


\textbf{Sanity-checking attribution methods}.
Several attribution methods~\cite{murdoch2019definitions} are based on input gradients and explicitly or implicitly assume an appropriately modified version of \premise. 
For example, Integrated Gradients~\cite{sundararajan2017axiomatic} aggregate input gradients of linearly interpolated points, SmoothGrad~\cite{smilkov2017smoothgrad} averages input gradients of points perturbed using gaussian noise, and Guided Backprop~\cite{springenberg2015striving} modifies input gradients by zeroing out negative values at every layer during backpropagation.
Surprisingly, unlike vanilla input gradients, popular methods that output attributions with better visual quality fail simple sanity checks that are indeed expected out of any valid attribution method~\cite{adebayo2018sanity,kindermans2019reliability}. 
On the other hand, while vanilla input gradients pass simple sanity checks,~\citet{Hooker2019ABF} suggest that they produce estimates of feature importance that are no better than a random designation of feature importance.


\textbf{Do input gradients satisfy assumption~\premise?}
Since \premise is necessary for input gradients attributions to accurately reflect model behavior, 
we introduce an  evaluation framework, \evalname, to analyze whether input gradient attributions satisfy assumption \premise~on real-world datasets. While \evalname\ adopts the remove-and-retrain (\texttt{ROAR}) methodology~\cite{Hooker2019ABF}, \evalname~is more appropriate for testing the validity of assumption \premise because it directly compares top-ranked features against bottom-ranked features.  
We apply~\evalname~to evaluate input gradient attributions of MLPs \& CNNs trained on multiple image classification datasets. 
Consistent with the message in \citet{Hooker2019ABF}, our experiments indicate that input gradients of standard models (i.e., trained on original data) can  grossly violate \premise (see Section~\ref{sec:real}). Furthermore, we also observe that unlike standard models, adversarially trained models~\cite{madry2018towards} that are robust to $\ell_2$ and $\ell_{\infty}$ perturbations satisfy \premise in a consistent manner.

\textbf{Probing input gradient attributions using~\semirealname}. 
Our empirical findings mentioned above strongly suggest that standard models grossly violate \premise.
However, without knowledge of  ground-truth discriminative features learned by models trained on real data, {\em conclusively} testing~\premise remains elusive. 
In fact, this is a key shortcoming of the remove-and-retrain (\texttt{ROAR}) framework.
\begin{figure*}[t]
	\centering
	\includegraphics[width=\linewidth]{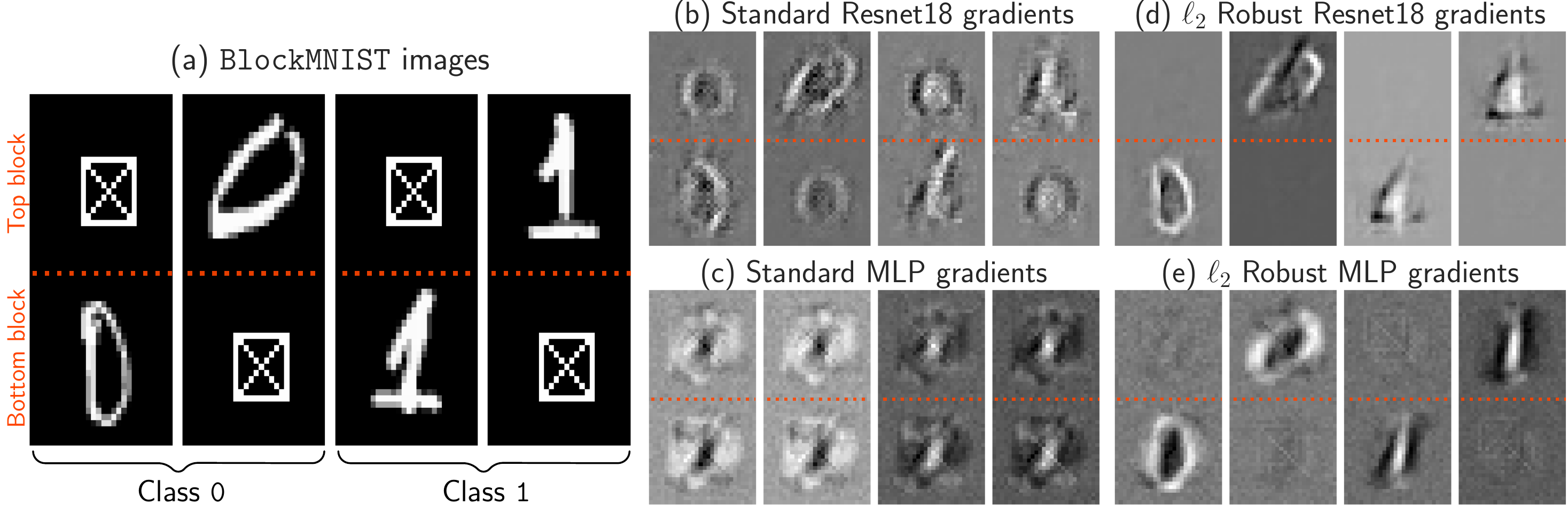}
	\caption{Experiments on \semirealname~dataset. (a)	Four representative images from class $0$ \& class $1$ in \semirealname~dataset; every image consists of a \emph{signal} and \emph{null} block that are randomly placed as the \emph{top} or \emph{bottom} block. The \emph{signal} block, containing the \texttt{MNIST} digit, determines the image class. The \emph{null} block, containing the square patch, does not encode any information of the image class. For these four images, subplots (b-e) show the input gradients of standard Resnet18, standard MLP, $\ell_2$ robust Resnet18 ($\epsilon$=$2$) and $\ell_2$ robust MLP ($\epsilon$=$4$) respectively.
		The plots clearly show that input gradients of standard \semirealname~models highlight the signal block \emph{and the non-discriminative null block}, thereby violating \premise.
		In contrast, input gradients of adversarially robust models exclusively highlight the signal block, suppress the null block, and satisfy \premise. 
		Please see~\Cref{sec:semireal} for details.
	}
	\label{fig:semireal_robust}
	\vspace{-15px}
\end{figure*}
So, to further verify and better understand our empirical findings, we introduce an \texttt{\mnist}-based semi-real dataset, \semirealname, 
that \emph{by design} encodes {a priori} knowledge of ground-truth discriminative features.
\semirealname~is based on the principle that for different inputs, discriminative and non-discriminative features may occur in different parts of the input. 
For example, in an object classification task, the object of interest can occur in different parts of the image (e.g., top-left, center, bottom-right etc.) for different images.
As shown in~\Cref{fig:semireal_robust}(a),~\semirealname~images consist of a \emph{signal} block and a \emph{null} block that are randomly placed at the top or bottom.
The \emph{signal} block contains the \texttt{MNIST} digit that determines the class of the image, whereas the \emph{null} block, contains a square patch with two diagonals that has no information about the label. 
This a priori knowledge of ground-truth discriminative features in \semirealname~data allows us to (i) validate our empirical findings vis-a-vis input gradients of standard and robust models (see~\cref{fig:semireal_robust}) and (ii) identify \emph{feature leakage} as a reason that potentially explains why input gradients violate \premise in practice. 
Here, feature leakage refers to the phenomenon wherein given an instance, its input gradients highlight the location of discriminative features in the given instance \emph{as well as} in other instances that are present in the dataset.
For example, consider the first~\semirealname~image in~\cref{fig:semireal_robust}(a), in which the signal is placed in the bottom block. 
For this image, as shown in~\cref{fig:semireal_robust}(b,c), input gradients of standard models {incorrectly} highlight the top block \emph{because} there are \emph{other} instances in the \semirealname~dataset which have signal in the top block.

\textbf{Rigorously demonstrating feature leakage}.
In order to concretely verify as well as understand feature leakage more thoroughly, we design a simplified version of \semirealname~that is amenable to theoretical analysis. 
On this dataset, we first rigorously demonstrate that input gradients of standard one-hidden-layer MLPs exhibit feature leakage in the infinite-width limit and then discuss how feature leakage results in input gradient attributions that clearly violate assumption \premise.



\textbf{Paper organization}: Section~\ref{sec:related} discusses related work and section~\ref{sec:prelim} presents our evaluation framework,~\evalname, to test assumption \premise.~\Cref{sec:real} employs~\evalname~to evaluate input gradient attributions on four image classification datasets.~\Cref{sec:semireal} analyzes~\semirealname~data to differentially characterize input gradients of standard and robust models using feature leakage.~\Cref{sec:synth} provides theoretical results on a simplified version on~\semirealname~that shed light on how feature leakage results in input gradients that violate assumption \premise.
Our code, along with the proposed datasets, is publicly available at \href{https://github.com/harshays/inputgradients}{https://github.com/harshays/inputgradients}.

\vspace{-11px}
\section{Related work}
\label{sec:related}
\vspace{-7px}
Due to space constraints, we only discuss directly related work and defer the rest to~\Cref{appendix:related}. 

\textbf{Sanity checks for explanations}.
Several explanation methods that provide feature attributions are often primarily evaluated using inherently subjective visual assessments~\cite{simonyan2013deep,smilkov2017smoothgrad}.
Unsurprisingly, recent ``sanity checks'' show that sole reliance on visual assessment is misleading, as attributions can lack fidelity and inaccurately reflect model behavior.
\citet{adebayo2018sanity} and~\citet{kindermans2019reliability} show that unlike input gradients~\cite{baehrens2010explain}, other popular methods---guided backprop~\cite{springenberg2014striving}, gradient $\odot$ input~\cite{shrikumar2016not}, integrated gradients~\cite{sundararajan2017axiomatic}---output explanations which lack fidelity on image data, as they remain invariant to model and label randomization.
Similarly,~\citet{yang2019benchmarking} use custom image datasets to show that several explanation methods are more likely to produce false positive explanations than vanilla input gradients. 
Moreover, several explanation methods based on modified backpropagation do not pass basic sanity checks~\cite{pmlr-v119-sixt20a,NEURIPS2019_a7471fdc,ancona2018towards}.
To summarize, well-known gradient-based attribution methods that seek to mitigate gradient saturation~\cite{sundararajan2017axiomatic,sundararajan2016gradients}, discontinuity~\cite{shrikumar2017learning}, and visual noise~\cite{springenberg2014striving} surprisingly fare worse than vanilla input gradients on multiple sanity checks.

 
\textbf{Evaluating explanation fidelity}. 
The black-box nature of neural networks necessitates frameworks that evaluate the fidelity or ``correctness'' of post-hoc explanations \emph{without} knowledge of ground-truth features learned by trained models.
Modification-based evaluation frameworks~\cite{10.1371/journal.pone.0130140, 7552539, arras-etal-2019-evaluating} gauge explanation fidelity by measuring the change in model performance after masking input coordinates that a given explanation method considers most (or least) important.
However, due to distribution shifts induced by input modifications, one cannot  \emph{conclusively} attribute changes in model performance to the fidelity of instance-specific explanations~\cite{Tomsett2020SanityCF}.
The remove-and-retrain (\texttt{ROAR}) framework~\cite{Hooker2019ABF} accounts for distribution shifts by evaluating the performance of models \emph{retrained} on train data masked using post-hoc explanations.
Surprisingly, contrary to findings obtained via sanity checks, experiments with the \texttt{ROAR} framework show that multiple attribution methods, \emph{including} vanilla input gradients, are no better than model-independent \emph{random} attributions {that lack explanatory power}~\cite{Hooker2019ABF}. 
Therefore, motivated by the central role of vanilla input gradients in attribution methods, we augment the \texttt{ROAR} framework to understand when and why input gradients violate assumption \premise.

\textbf{Effect of adversarial robustness}.
Adversarial training~\cite{madry2018towards} not only leads to robustness to $\ell_p$ adversarial attacks~\cite{tramer2020adaptive}, but also leads to perceptually-aligned feature representations~\cite{santurkar2019image}, and improved visual quality of input gradients~\cite{Ross2018ImprovingTA}. 
Recent works hypothesize that adversarial training improves the visual quality of input gradients by suppressing irrelevant features~\cite{Kim2019WhyAS} and promoting sparsity and stability~\cite{Chalasani2020ConciseEO} in explanations.
\citet{Kim2019BridgingAR} use the \texttt{ROAR} framework to conjecture that adversarial training ``tilts'' input gradients to better align with the data manifold. 
In this work, we use experiments on real-world data and theory on data with features known \emph{a priori} in order to differentially characterize input gradients of standard and robust models vis-a-vis assumption \premise.

\section{\texttt{DiffROAR} evaluation framework}
\label{sec:prelim}
In this section, we introduce our evaluation framework, \texttt{DiffROAR}, to probe the extent to which instance-specific explanations, or feature attributions, highlight discriminative features in practice.
Specifically, our framework, \texttt{DiffROAR}, builds upon the remove-and-retrain (\texttt{ROAR}) methodology~\cite{Hooker2019ABF} to test whether feature attribution methods satisfy assumption~\premise on real-world datasets. 

\textbf{Setting}. We consider the standard classification setting;
Each data point $(x^{(i)},y^{(i)})$, where instance $x^{(i)} \in \R^d$ and label $y^{(i)} \in \Y$ for some label set $\Y$, is drawn independently from a distribution $\D$ on $\R^d \times \Y$. 
Given dataset $\{(x^{(i)},y^{(i)})\}$ where $i\in [n] \defeq \set{1,\cdots,n}$, $x^{(i)}_j$ denotes the $j^{\textrm{th}}$ coordinate of $x^{(i)}$. 
Note that we also refer to the $d$ coordinates of instance $x^{(i)}$ as  \emph{features} interchangeably. 

\textbf{Attribution schemes}. 
A \emph{feature attribution} scheme $\A: \R^d \rightarrow \set{\sigma: \sigma \mbox{ is a permutation of } [d]}$ maps a $d$-dimensional instance $x$ to a permutation, or ordering, $\A(x): [d] \rightarrow [d]$ of its coordinates.
For example, the \emph{input gradient attribution} scheme takes as input instance $x$ \& predicted label $\hat{y}$ and outputs an ordering $[d]$ that ranks coordinates in decreasing order of their input gradient magnitude. 
 That is, coordinate $j$ is ranked ahead of coordinate $k$ if the magnitude of the $j^{\textrm{th}}$ coordinate of $\nabla_x \Logit_{\theta}(x,\hat{y})$ is larger than that of the $k^{\textrm{th}}$ coordinate. 

\textbf{Unmasking schemes}.
Given instance $x$ and a subset $S \subseteq [d]$ of coordinates, the \emph{unmasked} instance $x^S$ zeroes out all coordinates that are not in subset $S$: $x^S_j=x_j$ if $j \in S$ and $0$ if $j \notin S$.
An \emph{unmasking scheme} $A: \R^d\rightarrow \set{S: S\subseteq [d]}$ simply maps instance $x$ to a subset $A(x) \subseteq[d]$ of coordinates that can be used to obtain unmasked instance $x^{A(x)}$.
Any attribution scheme $\A$ naturally induces \emph{top-$k$} and \emph{bottom-$k$} unmasking schemes, $\At{k}$ and $\Ab{k}$, which output $k$ coordinates with the top-most and bottom-most attributions in $\A(x)$ respectively. In other words, given attribution scheme $\A$ and level $k$, the top-k and bottom-k unmasking schemes, $\At{k}$ and $\Ab{k}$, can be defined as follows:
\begin{wrapfigure}{r}{0.35\linewidth}
	\centering
	\vspace{-8px}
	\includegraphics[width=\linewidth]{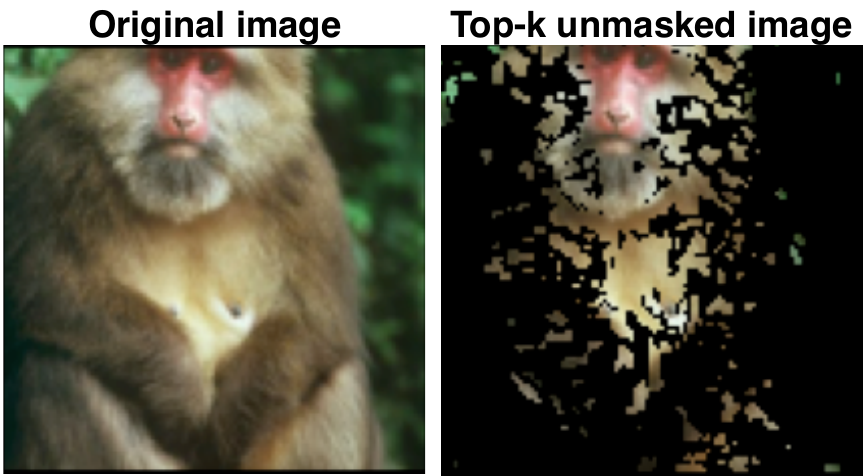}
	\vspace{-12px}
	\caption{Pictorial example of a top-$25\%$ unmasked image.}
	\label{fig:unmasking}
	\vspace{-12px}
\end{wrapfigure}
\vspace{-7px}
\begin{align*}
\At{k}(x) &\defeq \set{\A(x)_{j}: j \leq k}, \\
\Ab{k}(x) &\defeq \set{\A(x)_{j}: d-k < j \leq d}.
\end{align*}
For example,~\Cref{fig:unmasking} depicts an image $x$ and its top-$k$ unmasked variant $x^{\At{k}(x)}$. 
In this case, the attribution scheme $\A$ assigns higher rank to pixels in the foreground.
So, the top-$25\%$ unmasking operation, $x^{\At{25\%}(x)}$,  highlights the monkey by retaining pixels with top-$25\%$ attribution ranks and zeroing out the remaining pixels that correspond to the green background.

\textbf{Predictive power of unmasking schemes}. The \emph{predictive power} of an unmasking scheme $A$ with respect to model architecture $M$ (e.g., resnet18) can be defined as the best classification accuracy that can be attained by training a model with  architecture $M$ on unmasked instances that are obtained via unmasking scheme $A$. More formally, it can defined as follows:
	\begin{align*}
		\predm{A} \;\defeq \sup_{f \in M, f: \R^d \rightarrow \Y} \ED{\ind{f(x^{A(x)})=y}}.
	\end{align*}	
	
Due to masking-induced distribution shifts, models with architecture $M$ that are trained using original data cannot be plugged in to estimate $\predm{A}$. 
The \texttt{ROAR} framework~\cite{Hooker2019ABF} sidesteps this issue by \emph{retraining} models on unmasked data, as similar model architectures tend to learn ``similar'' classifiers~\cite{pmlr-v119-hacohen20a,pmlr-v44-li15convergent,NEURIPS2019_b432f34c,NEURIPS2020_6cfe0e61}.
Therefore, we employ the \texttt{ROAR} framework to estimate $\predm{A}$ in two steps.
First, we use unmasking scheme $A$ to obtain \emph{unmasked} train and test datasets that comprise data points of the form $(x^{A(x)}, y)$.
Then, we \emph{retrain} a new model with the same architecture $M$ on unmasked train data and evaluate its accuracy on unmasked test data. 

\textbf{\evalname~evaluation metric to test assumption~\premise}.
Recall that an attribution scheme $\A$ maps an instance $x$ to a permutation of its coordinates that reflects the order of \emph{estimated} importance in model prediction. 
An attribution scheme that satisfies assumption~\premise must place coordinates that are more important for model prediction higher up in the the attribution order.
More formally, given attribution scheme $\A$, architecture $M$ and level $k$, we define \evalname~as the difference between the predictive power of top-$k$ and bottom-$k$ unmasking schemes, $\At{k}$ and $\Ab{k}$:
\begin{align}
	\texttt{DiffROAR}_M(\A,k) = \predm{\At{k}} - \predm{\Ab{k}}  \label{eqn:diffroar}
\end{align}

\textbf{Interpreting the \evalname~metric}. The {sign} of the \evalname~metric indicates whether the given attribution scheme satisfies or violates assumption~\premise.
For example, $\texttt{DiffROAR}_M(\A,\cdot) < 0$ implies that $\A$ violates assumption \premise~, as coordinates with \emph{higher} attribution ranks have \emph{worse} predictive power with respect to architecture $M$.
Similarly, the {magnitude} of the \evalname~metric quantifies the extent to which the ordering in attribution scheme $\A$ separates the most and least discriminative coordinates into two disjoint subsets.
For example, a \emph{random} attribution scheme $\A_{r}$, which outputs attributions $\A_r(x)$ chosen uniformly at random from all permutations of $[d]$, neither highlights nor suppresses discriminative features; $\mathbb{E}[\texttt{DiffROAR}_M(\A_r, k)]=0$ for any architecture $M$.

\textbf{On testing assumption~\premise}.
To verify~\premise for a given attribution scheme $\A$, it is necessary to evaluate whether input coordinates with \emph{higher} attribution rank are \emph{more} important for model prediction than coordinates with \emph{lower} rank. 
Consequently, the \texttt{ROAR}-based metric in~\citet{Hooker2019ABF}, which essentially computes the top-$k$ predictive power, is not sufficient to test whether attribution methods satisfy assumption~\premise.
Therefore, as discussed above, \evalname~tests \premise~by comparing the top-$k$ predictive power, $\predm{\At{k}}$, to the bottom-$k$ predictive power, $\predm{\Ab{k}}$, using multiple values of $k$.  

\section{Testing assumption~\premise on image classification benchmarks}
\label{sec:real}

In this section, we use~\evalname~to evaluate whether input gradient attributions of standard and adversarially robust MLPs and CNNs trained on four image classification benchmarks satisfy assumption~\premise. We first summarize the experiment setup and then describe key empirical findings.

\textbf{Datasets and models}.
We consider four benchmark image classification datasets: SVHN~\cite{netzer2011reading}, Fashion MNIST~\cite{xiao2017fashion}, \cifar~\cite{krizhevsky2009learning} and \imnet~\cite{imagenet_cvpr09}. 
\imnet\ is an open-sourced variant (\url{https://github.com/MadryLab/robustness/}) of Imagenet~\cite{imagenet_cvpr09}, with $80,000$ images grouped into $10$ super-classes. 
\imnet~enables us to test assumption \premise~on Imagenet without the computational overload of training models on the $1000$-way ILSVRC classification task ~\cite{ILSVRC15}. 
We evaluate input gradient attributions of standard and adversarially trained two-hidden-layer MLPs and Resnets~\cite{He2015}.
We obtain $\ell_2$ and $\ell_{\infty}$ $\epsilon$-robust models with perturbation budget $\epsilon$ using PGD adversarial training~\cite{madry2018towards}.
Unless mentioned otherwise, we train models using stochastic gradient descent (SGD), with momentum $0.9$, batch size $256$, $\ell_2$ regularization $0.0005$ and initial learning rate $0.1$ that decays by a factor of $0.75$ every $20$ epochs. Additionally, we use standard data augmentation and train models for at most $500$ epochs, stopping early if cross-entropy loss on training data goes below $0.001$. 
\Cref{app:expt_deets} provides additional details about the datasets and trained models.\footnote{Code publicly available at \url{https://github.com/harshays/inputgradients}}


\textbf{Estimating \evalname~on real data}. We compute the evaluation metric, $\texttt{DiffROAR}_M(\A,k)$, on real datasets in four steps, as follows.
First, we train a standard or robust model with architecture $M$ on the original dataset and obtain its input gradient attribution scheme $\A$. 
Second, as outlined in~\Cref{sec:prelim}, we use attribution scheme $\A$ and level $k$ (i.e., fraction of pixels to be unmasked) to extract the top-$k$ and bottom-$k$ unmasking schemes: $\At{k}$ and $\Ab{k}$.
Third, we apply $\At{k}$ and $\Ab{k}$ on the original train \& test datasets to obtain top-$k$ and bottom-$k$ unmasked datasets respectively. 
Finally, to compute $\texttt{DiffROAR}_M(\A,k)$ via eq.~\eqref{eqn:diffroar}, we estimate top-$k$ and bottom-$k$ predictive power, $\predm{\At{k}}$ and $\predm{\Ab{k}}$, by \emph{retraining new models} with architecture $M$ on top-$k$ and bottom-$k$ unmasked datasets respectively.
Also, note that we (a) average the \evalname~metric over five runs for each model and unmasking fraction or level $k$ and (b) unmask individual image pixels without grouping them channel-wise.

\begin{figure*}[t]
	\centering
	\includegraphics[width=0.99\linewidth]{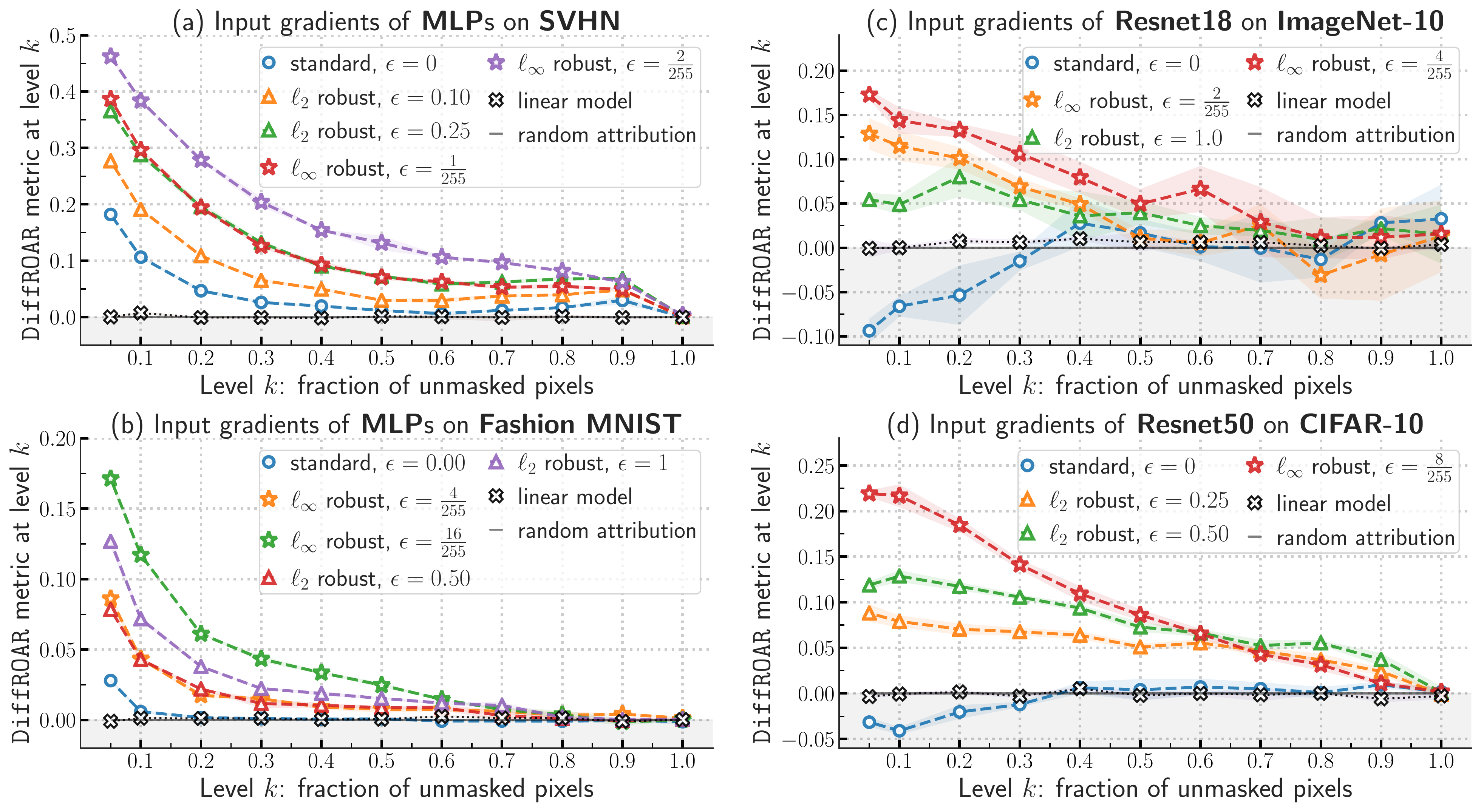}
	\vspace{-7px}
	\caption{
	\evalname~plots for input gradient attributions of standard and adversarially robust two-hidden-layer MLPs on (a) SVHN \& (b) Fashion MNIST, (c) Resnet18 on \imnet~and (d) Resnet50 on~\cifar.
	Subplot (a) indicates that adversarially robust MLPs consistently and considerably outperform standard MLPs on the~\evalname~metric for all $k < 100\%$.
	Subplot (b) shows that for most unmasking fractions $k$, standard MLPs trained on Fashion MNIST, unlike robust MLPs, fare no better than model-agnostic random attributions and input-agnostic attributions of linear models.
	Subplots (c) and (d) show that when $k<40\%$, standard Resnet models trained on~\cifar~and~\imnet~grossly violate~\premise, thereby implying that coordinates with top-most gradient attribution rank have worse predictive power than coordinates with bottom-most rank. 
	In stark contrast, input gradients of Resnets that are robust to $\ell_2$ and $\ell_{\infty}$ adversarial perturbations satisfy assumption \premise reasonably well.
	We observe that increasing the perturbation budget $\epsilon$ during adversarial training amplifies the magnitude of \evalname~for every $k$ across all four image classification benchmarks.
	} 
 \label{fig:real}
 \vspace{-15px}
\end{figure*}
\textbf{Experiment setup}.
Now, we analyze the~\evalname~metric as a function of  the unmasking fraction $k \in \{5, 10, 20, \ldots, 100\}\%$ in order to evaluate whether input gradient attributions of models trained on four image classification benchmarks satisfy assumption~\premise. 
In particular, as shown in~\Cref{fig:real}, we use~\evalname~to analyze input gradients of standard and adversarially robust two-hidden-layer MLPs on SVHN \& Fashion MNIST, Resnet18 on \imnet, and Resnet50 on~\cifar.
In order to calibrate our findings, we compare input gradient attributions of these models to two natural baselines: model-agnostic \emph{random} attributions and input-agnostic attributions of linear models. 

\textbf{Input gradients of standard models}.
Input gradient attributions of standard MLPs trained on SVHN satisfy assumption~\premise, as the \evalname~metric  in~\Cref{fig:real}(a) is  positive for all values of level $k$ < 100\%.
However, in~\Cref{fig:real}(b), the \evalname~curves of standard MLPs trained on Fashion MNIST indicate that input gradient attributions, consistent with findings in~\citet{Hooker2019ABF}, can fare no better than model-agnostic {random} attributions and input-agnostic attributions of linear models vis-a-vis assumption~\premise.
Furthermore, and rather surprisingly, the shaded area in~\Cref{fig:real}(c) and~\Cref{fig:real}(d) 
shows that when level $k$ < 40\%, \evalname~curves of standard Resnets trained on CIFAR-10 and Imagenet-10 are consistently \emph{negative} and perform considerably worse than model-agnostic and input-agnostic baseline attributions.
These results strongly suggest that on CIFAR-10 and Imagenet-10, input gradients of standard Resnets grossly violate assumption \premise and  suppress discriminative features. In other words, coordinates with \emph{larger} gradient magnitude have \emph{worse} predictive power than coordinates with \emph{smaller} gradient  magnitude.

\textbf{Input gradients of robust models}.
Models that are $\epsilon$-robust to $\ell_2$ and $\ell_{\infty}$ adversarial perturbations fare considerably better than standard models on the \evalname~metric.
For example, in~\Cref{fig:real}(a), when level $k$ equals $10$\%,  robust MLPs trained on SVHN outperform standard MLPs on the~\evalname~metric by roughly $10\text{-}30\%$.
The \evalname~curves of adversarially robust MLPs in~\Cref{fig:real} are positive at every level $k$ < $100$\%, which strongly suggests that input gradient attributions of robust MLPs satisfy assumption~\premise. 
Similarly, robust resnet50 models trained on \cifar~and \imnet~satisfy assumption \premise reasonably well and, unlike standard resnet50 models, starkly highlight discriminative features.
Furthermore, we observe that increasing the perturbation budget $\epsilon$ in $\ell_2$ or $\ell_{\infty}$ PGD adversarial training~\cite{madry2018towards} amplifies the magnitude of \evalname~across $k$ and for all four datasets.
That is, the adversarial perturbation budget $\epsilon$ determines the extent to which input gradients differentiates the most and least discriminative coordinates into two disjoint subsets. 

\textbf{Additional results}. 
In~\Cref{appendix:real}, we show that our~\evalname~results are robust to choice of model architecture \& SGD hyperparameters during retraining and also hold for input gradients taken with respect to cross-entropy loss. 
Additionally, while~\evalname~\emph{without retraining} gives qualitatively similar results, they are not as consistent across architectures as with retraining, particularly for small unmasking fraction $k$ that induce non-trivial distribution shifts.
\section{Analyzing input gradient attributions using~\semirealname~data}
\label{sec:semireal}
\vspace{-10px}
To verify whether input gradients satisfy assumption \premise~more thoroughly, we introduce and perform experiments on \semirealname, an \texttt{MNIST}-based dataset that \emph{by design} encodes a priori knowledge of ground-truth discriminative features.

\textbf{\semirealname~dataset design}: 
The design of the \semirealname~dataset is based on two intuitive properties of real-world object classification tasks: (i) for different images, the object of interest may appear in different parts of the image (e.g., top-left, bottom-right); (ii) the object of interest and the rest of the image often share low-level patterns such as edges that are not informative of the label on their own. 
We replicate these aspects in~\semirealname~instances, which are vertical concatenations of two $28 \times 28$ \emph{signal} and \emph{null} image blocks that are randomly placed at the top or bottom with equal probability.
The \emph{signal} block is an \texttt{MNIST} image of digit $0$ or digit $1$, corresponding to class $0$ or $1$ of the \semirealname~image respectively. 
On the other hand, the \emph{null} block in every \semirealname~image, independent of its class, contains a square patch made of two horizontal, vertical, and slanted lines, as shown in~\Cref{fig:semireal_robust}(a).
It is important to note that unlike the \texttt{MNIST} signal block that is fully predictive of the class, the non-discriminative null block contains no information about the class.
Standard as well as adversarially robust models trained on \semirealname~data attain 99.99\% test accuracy, thereby implying that model predictions are indeed based solely on the signal block for any given instance. 
We further verify this by noting that the predictions of trained model remain unchanged on almost every instance even when all pixels in the null block are set to zero.

\textbf{Do standard and robust models satisfy \premise?}
As discussed above, unlike the null block that has no task-relevant information, the \texttt{MNIST} digit in the signal block entirely determines the class of any given \semirealname~image.
Therefore, in this setting, we can restate assumption \premise~ as follows: \emph{Do input gradient attributions highlight the signal block over the null block?}
Surprisingly, as shown in~\Cref{fig:semireal_robust}(b,c), input gradient attributions of standard MLP and Resnet18 models highlight the signal block \emph{as well as} the non-discriminative null block. 
In stark contrast, subplots (d) and (e) in~\Cref{fig:semireal_robust} show that input gradient attributions of $\ell_2$ robust MLP and Resnet18 models exclusively highlight \texttt{MNIST} digits in the signal block and clearly suppress the square patch in the null block. 
These results validate our findings on real-world datasets by showing that unlike standard models, adversarially robust models satisfy \premise~on \semirealname~data. 

\textbf{Feature leakage hypothesis}:
Recall that the discriminative signal block in \semirealname~images is randomly placed at the top or bottom with equal probability.
Given our results in~\Cref{fig:semireal_robust}, we hypothesize that when discriminative features vary across instances (e.g., signal block at top vs. bottom), input gradients of standard models not only highlight instance-specific features but also \emph{leak} discriminative features {from} other instances. We term this hypothesis \emph{feature leakage}.
\begin{wrapfigure}{r}{0.58\linewidth}
	\centering
	\vspace{-13px}
	\includegraphics[width=\linewidth]{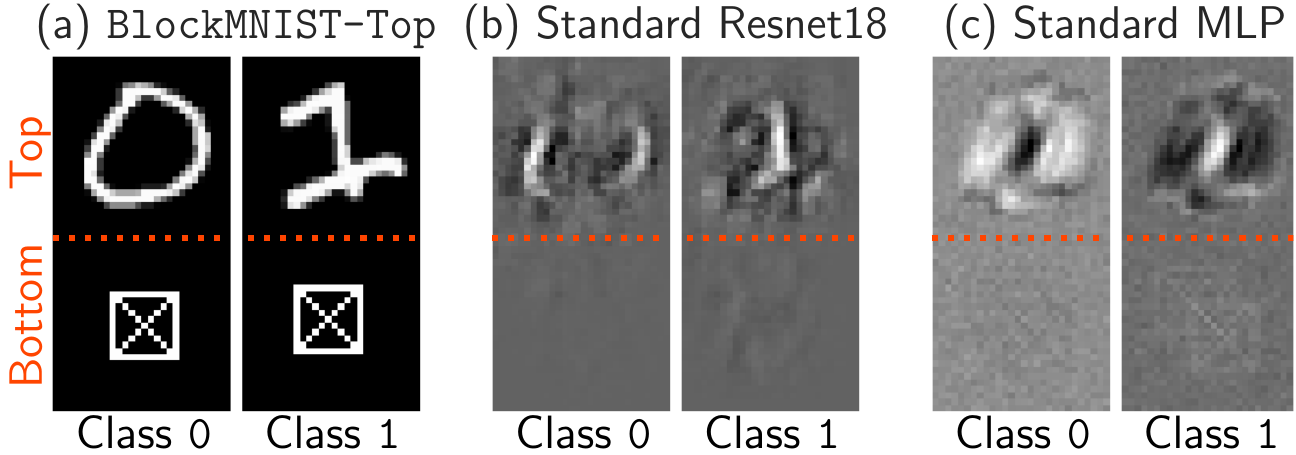}
	\vspace{-16px}
	\caption{
	(a) In \semirealnameone~images, the signal \& null blocks are fixed at the top \& bottom respectively.
	In contrast to results on \semirealname~in~\cref{fig:semireal_robust}, input gradients of standard (b) Resnet18 and (c) MLP trained on~\semirealnameone~highlight discriminative features in the signal block, suppress the null block, and satisfy \premise. 
	}
	\label{fig:semireal_leakage}
	\vspace{-10px}
\end{wrapfigure}
To test our hypothesis, we leverage the modular structure in \semirealname~to construct a slightly modified version, \semirealnameone, wherein the location of the \texttt{MNIST} signal block is fixed at the top for all instances (see~\cref{fig:semireal_leakage}). 
In this setting, in contrast to results on \semirealname, input gradients of \emph{standard} Resnet18 and MLP models trained on~\semirealnameone~ satisfy assumption \premise. 
Specifically, when the signal block is fixed at the top, input gradient attributions in~\Cref{fig:semireal_leakage}(b, c) clearly highlight the signal block and suppress the null block, thereby supporting our feature leakage hypothesis.
Based on our \semirealname~experiments, we believe that understanding \emph{how} adversarial robustness mitigates feature leakage is an interesting direction for future work.

\textbf{Additional results}. 
In~\Cref{app:leakage}, we 
(i) visualize input gradients of several \semirealname~and \semirealnameone~images, 
(ii) introduce a quantitative proxy metric to compare feature leakage between standard and robust models, 
(iii) show that our findings are fairly robust to the choice and number of classes in~\semirealname~data,
and (iv) evaluate feature leakage in five feature attribution methods. 
We also provide experiments that falsify hypotheses vis-a-vis input gradients and assumption \premise that we considered in addition to feature leakage.

%

\section{{Feature leakage in input gradient attributions}}
\label{sec:synth}
To understand the extent of feature leakage more thoroughly, we introduce a simplified version of the \semirealname~dataset that is amenable to theoretical analysis.
We rigorously show that input gradients of standard one-hidden-layer MLPs do not differentiate instance-specific features from other task-relevant features that are not pertinent to the given instance.

\textbf{Dataset}: 
Given dimension of each block $\dtilde$, feature vector $u^* \in \R^{\dtilde}$ with $\norm{u^*}=1$, number of blocks $d$ and noise parameter $\eps$, we will {construct} input {instances} of dimension $\dtilde \cdot d$. More concretely, a sample $(x,y) \in \R^{\dtilde \cdot d} \times \set{\pm 1}$ from the distribution $\D$ is generated as follows:
\begin{align}
	y &= \pm 1 \mbox{ with probability 0.5} \text{ and }  \label{eqn:synth} \nonumber \\
	x &= \begin{bmatrix}
		\eps g_1, & \eps g_2, & \ldots, & y u^* + \eps g_j, & \ldots, & \eps g_d
	\end{bmatrix} \mbox{ with } j \mbox{ chosen at random from } [\sfrac{d}{2}] 
\end{align}
where each $g_i \in \R^{\dtilde}$ is drawn uniformly at random from the unit ball. For simplicity, we take $d$ to be even so that $\sfrac{d}{2}$ is an integer.
We can think of each $x$ as a concatenation of $d$ $\dtilde$-dimensional blocks $\{x_1,\ldots,x_d\}$.
The first $\sfrac{d}{2}$ blocks, $\set{1,\ldots,\sfrac{d}{2}}$, are \emph{task-relevant}, as every example $(x,y)$ contains an instance-specific \emph{signal} block $x_i= y u^* + \eps g_i$ for some $i \in [\sfrac{d}{2}]$ that is informative of its label $y$.
Given instance $x$, we use $j^*(x)$ to denote the unique instance-specific signal block such that $x_{j^*(x)}=y u^* + \eps g_{j^*(x)}$. 
On the other hand, \emph{noise} blocks $\set{\sfrac{d}{2}+1,\ldots,d}$ do not contain task-relevant signal for any instance $x$.
At a high level, the instance-specific signal block $j^*(x)$ and noise blocks $\set{\sfrac{d}{2}+1,\ldots,d}$ in instance $x$ correspond to the discriminative \texttt{MNIST} digit and the null square patch in~\semirealname~images respectively.
For example, each row in~\Cref{fig:21_synth}(a) illustrates an instance $x$ where $d=10, \dtilde=1, \eta =0$ and $u^*=1$. 



\textbf{Model}: We consider one-hidden layer MLPs with ReLU nonlinearity in the infinite-width limit. More concretely, for a given width $m$, the network is parameterized by $R \in \R^{m \times \dtilde \cdot d}, b \in \R^m$ and $w \in \R^m$. Given an input instance $(x,y) \in \R^{\dtilde d} \times \set{\pm 1}$, the output score (or logit) $f$ and cross-entropy (CE) loss $\L$ are given by:
\begin{align*}
f((w,R,b),x) \defeq \iprod{w}{\phi\left(Rx+b\right)}, \quad 
\L((w,R,b),(x,y)) \defeq \log\left(1+\exp\left(-y \cdot f((w,R,b),x)\right)\right).
\end{align*}
where $\phi(t) \defeq \max(0,t)$ denotes the ReLU function.
A remarkable set of recent results~\cite{mei2018mean,chizat2018global,rotskoff2018trainability,sirignano2021mean} show that as $m\rightarrow \infty$, the training procedure is equivalent to gradient descent ({GD}) on an infinite dimensional Wasserstein space.
In the Wasserstein space, the network can be interpreted as a probability distribution $\nu$ over $\R \times \R^{\dtilde \cdot d} \times \R$ with output score $f$ and cross entropy loss $\L$ defined as:
\begin{align}\label{eqn:loss-fcn}
	f(\nu,x) \defeq \Enu{w \phi\left(\iprod{r}{x}+b\right)}, \quad 
	\L(\nu,(x,y)) \defeq \log\left(1+\exp\left(-y \cdot f(\nu,x)\right)\right).
\end{align}
\textbf{Theoretical analysis}:
Our approach leverages the recent result in~\citet{chizat2020implicit}, which shows that if GD in the Wasserstein space $\W^2\left(\R \times \R^{\dtilde d} \times \R\right)$ on~$\ED{\L(\nu,(x,y))}$ converges, it does so to a max-margin classifier given by:
\begin{align}\label{eqn:F1-max}
	\nu^* \defeq \argmax_{{\nu \in \P\left(\Sddtilde\right)}} \min_{(x,y) \sim \D} y \cdot f(\nu,x),
\end{align}
where $\Sddtilde$ denotes the surface of the Euclidean unit ball in $\R^{\dtilde d+2}$, and $\P\left(\Sddtilde\right)$ denotes the space of probability distributions over $\Sddtilde$.
Intuitively, our main result shows that on any data point $(x,y) \sim \D$, the input gradient magnitude of the max-margin classifier $\nu^*$ is \emph{equal} over all task-relevant blocks $\set{1,\ldots,\sfrac{d}{2}}$ and zero on the remaining \emph{noise} blocks $\set{\sfrac{d}{2}+1,\ldots,d}$.
\begin{theorem}\label{thm:main}
	Consider distribution $\D$~\eqref{eqn:synth} with $\eta < \frac{1}{10 d}$. \emph{There exists a} max-margin classifier $\nu^*$ for $\D$ in Wasserstein space (i.e., training both layers of FCN with $m\rightarrow\infty$) given by~\eqref{eqn:F1-max}, such that for all $ \forall \; (x,y) \sim \D$: (i) $\norm{\left(\nabla_x \L(\nu^*,(x,y))\right)_j} = c > 0$ for every $j \in [d/2]$ and (ii) $\norm{\left(\nabla_x \L(\nu^*,(x,y))\right)_j}=0$ for every $j \in \set{d/2+1,\cdots,d}$, where $\left(\nabla_x \L(\nu^*,(x,y))\right)_j$ denotes the $j^\textrm{th}$ block of the input gradient $\nabla_x \L(\nu^*,(x,y))$.
\end{theorem}
Theorem~\ref{thm:main} guarantees the \emph{existence} of a max-margin classifier such that the input gradient magnitude for any given instance is (i) a non-zero constant on each of the first $\sfrac{d}{2}$ task-relevant blocks, and (ii) equal to zero on the remaining $\sfrac{d}{2}$ \emph{noise} blocks that do not contain any information about the label. 
However, input gradients fail at highlighting the \emph{unique instance-specific signal} block over the remaining \emph{task-relevant} blocks.
This clearly demonstrates feature leakage, as input gradients for any given instance also highlight task-relevant features that are, in fact, \emph{not specific} to the given instance. 
Therefore, input gradients of standard one-hidden-layer MLPs do not highlight instance-specific discriminative features and grossly violate assumption \premise.
In Appendix~\ref{app:conj}, we present additional results that demonstrate that adversarially trained one-hidden-layer MLPs can suppress feature leakage and satisfy assumption~\premise.

\textbf{Empirical results}: 
Now, we supplement our theoretical results by evaluating input gradients of linear models as well as standard \& robust one-hidden-layer ReLU MLPs with width $10000$ on the dataset shown in~\Cref{fig:21_synth}. 
Note that all models obtain 100\% test accuracy on this linearly separable dataset, a simplified version of BlockMNIST that is obtained via eq.~\ref{eqn:synth} with $d=10, \dtilde=1, \eta=0$ and $u^*=1$. 
Due to insufficient expressive power, linear models have input-agnostic gradients that suppress all five noise coordinates, but do not differentiate the instance-specific signal coordinate from the remaining task-relevant coordinates. 
Consistent with~\Cref{thm:main}, even standard MLPs, which are expressive enough to have input gradients that correctly highlight instance-specific coordinates, apply equal weight on all five task-relevant coordinates and violate~\premise due to feature leakage. 
On the other hand,~\Cref{fig:21_synth}(c) shows that the same MLP architecture, if robust to $\ell_{\infty}$ adversarial perturbations with norm $0.35$, satisfies \premise~ by clearly highlighting the instance-specific signal coordinate over all other noise \emph{and} task-relevant coordinates

\begin{figure}[t]
	\centering
	\includegraphics[width=\linewidth]{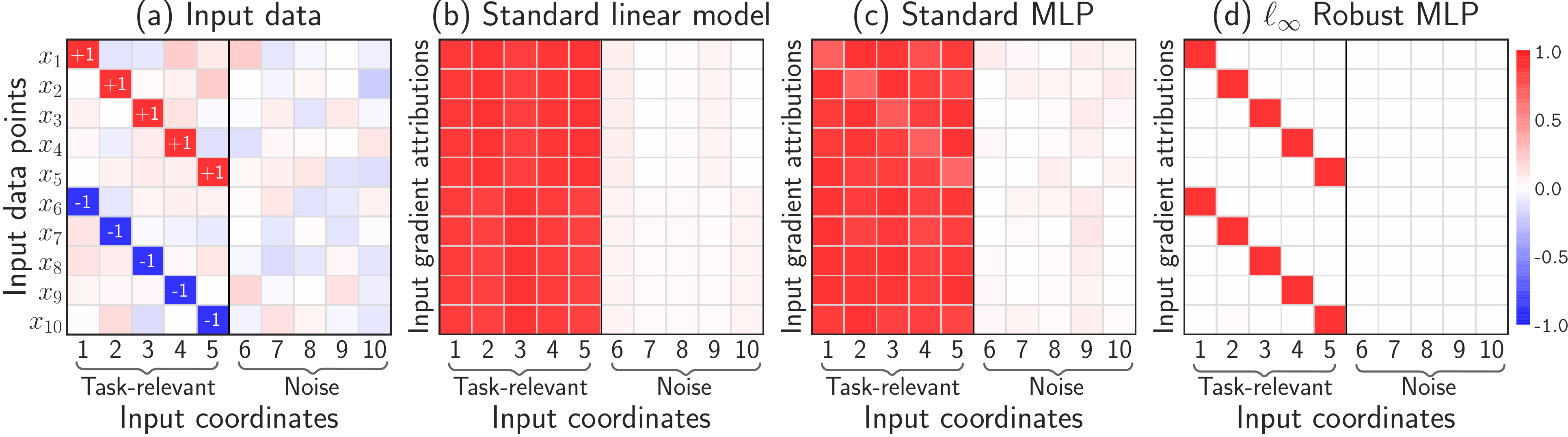}
	\vspace{-13px}
	\caption{
		Input gradients of linear models and standard \& robust MLPs trained on  data from eq.~\eqref{eqn:synth} with $d=10, \dtilde=1, \eta =0$ and $u^*=1$. 
		(a) Each row in corresponds to an instance $x$, and the highlighted coordinate denotes the signal block $j^*(x)$ \& label $y$.  
		(b) Linear models suppress noise coordinates but lack the expressive power to highlight instance-specific signal $j^*(x)$, as their input gradients in subplot (b) are identical across all examples.
		(c) Despite the expressive power to highlight instance-specific signal coordinate $j^*(x)$, input gradients of standard MLPs exhibit feature leakage (see \Cref{thm:main}) and violate \premise as well.
		(d) In stark contrast, input gradients of adversarially trained MLPs suppress feature leakage and starkly highlight instance-specific signal  coordinates $j^*(x)$. 
	}
	\label{fig:21_synth}
	\vspace{-5px}
\end{figure}

\section{Discussion and conclusion}
\label{sec:conc}
In this work, we took a three-pronged approach to investigate the validity of a key assumption made in several popular post-hoc attribution methods: 
\premise~\emph{coordinates with larger input gradient magnitude are more relevant for model prediction compared to coordinates with smaller input gradient magnitude}.
Through (i) evaluation on real-world data using our \evalname~framework, (ii) empirical analysis on \semirealname~data that encodes information of ground-truth discriminative features, and (iii) a rigorous theoretical study, we present strong evidence to suggest that standard models do not satisfy assumption \premise. In contrast, adversarially robust models satisfy \premise in a consistent manner. Furthermore, our analysis in~\Cref{sec:semireal} and~\Cref{sec:synth} indicates that \emph{feature leakage} sheds light on why input gradients of standard models tend to violate \premise. We provide additional discussion in~\Cref{appendix:discussion}.

This work exclusively focused on ``vanilla" input gradients due to their fundamental significance in \emph{feature attribution}. 
A similarly thorough investigation that analyzes other commonly-used attribution methods is an interesting avenue for future work.
Another interesting avenue for further analyses is to understand how adversarial training mitigates feature leakage in input gradient attributions.

\clearpage
\bibliographystyle{unsrtnat}
\bibliography{main.bib}



\newpage
\appendix
\appendixpage

The supplementary material is organized as follows. 
We first discuss additional related work~\Cref{sec:related}.
\Cref{appendix:discussion} provides additional discussion.
\Cref{appendix:real} describes additional experiments based on 
the \evalname~framework to analyze the fidelity of input gradient attributions on real-world datasets.
In~\Cref{appendix-semireal}, we provide additional experiments on feature leakage using~\semirealname-based datasets.
Then, \Cref{app:theory} contains the proof of Theorem 1 and~\Cref{appendix-adv-train} discusses the effect of adversarial training on input gradients of models that are adversarially trained on a simplified version of \semirealname~data.
We plan to open-source our trained models, code primitives, and Jupyter notebooks soon, which can be used to reproduce our empirical results.
\section{Additional related work}
\label{appendix:related}

In this section, we briefly describe works that analyze two properties of post-hoc instance-specific explanations that are related to explanation fidelity or ``correctness''.
In particular, we outline recent works that study the \emph{robustness} and \emph{practical utility} of instance-specific explanation methods.

\textbf{Robustness of explanations}: Several commonly used instance-specific explanation methods lack robustness in practice. 
\citet{Ghorbani2019InterpretationON} show that instance-specific explanations and exempler-based explanations are not robust to imperceptibly small adversarial perturbations to the input. 
\citet{Heo2019FoolingNN} show that instance-specific explanations are highly vulnerable to adversarial \emph{model} manipulations as well.
\citet{conf/nips/DombrowskiAAAMK19} show that explanations lack robustness to to \emph{arbitrary} manipulations and show that non-robustness stems from geometric properties of neural networks.  
\citet{Bansal2020SAMTS} show that explanation methods are considerably sensitive to method-specific hyperparameters such as sample size, blur radius, and random seeds.
Recent works promote robustness in explanations using smoothing~\cite{conf/nips/DombrowskiAAAMK19} or variants of adversarial training~\cite{Singh2020AttributionalRT,Lakkaraju2020RobustAS}

\textbf{Utility of explanations}:
A recent line of work propose evaluation frameworks to assess the practical utility of post-hoc instance-specific explanation methods via proxy downstream tasks.
\citet{Chu2020AreVE} employ a randomized controlled trial to show that using explanation methods as additional information does not improve \emph{human} accuracy on classification tasks.
More generally,~\citet{PoursabziSangdeh2018ManipulatingAM} analyze the effect of model transparency (e.g., number of input features, black-box vs. white-box) on the accuracy of human decisions with respect to the task and model.
Similarly,~\citet{Adebayo2020DebuggingTF} conduct a human subject study to show that subjects fail to identify defective models using attributions and instead primarily rely on model predictions. 
~\cite{Pruthi2020EvaluatingEH} formalize the ``value'' of explanations as the explanation utility (i.e., as side information) in a student-teacher learning framework.
In contrast to the works above, we propose an evaluation framework, \evalname, to evaluate the fidelity, or ``correctness'', of explanations in classification tasks.
In particular, using benchmark image classification tasks and synthetic data, we empirically and theoretically characterize input gradient attributions of standard as well as adversarially robust models.

\textbf{Stability of explanations}.
Explanation stability and explanation correctness (also known as explanation fidelity) are two distinct desirable properties of explanations~\cite{NEURIPS2019_a7471fdc}. That is, stability does not imply fidelity. For example, an input-agnostic constant explanation is stable but lacks fidelity. Conversely, fidelity does not imply stability---if the underlying model is itself unstable, then any correct high-fidelity explanation of that model must also be unstable.
\citet{Bansal2020SAMTS} and~\citet{chen2020shape} identify and explain why input gradients of adversarially trained models are more stable compared to those of standard models.
In contrast, our work focuses on identifying and explaining why input gradients of adversarially trained models have more fidelity compared to those of standard models. 
Furthermore, we also take the first step towards theoretically showing that adversarial robustness can provably improve input gradient fidelity in~\Cref{app:theory}.


\clearpage
\section{Additional discussion}
\label{appendix:discussion}

\textbf{Translation invariance in \semirealname~models}.
Intuitively, CNNs are translation-invariant only if the object of interest is not closer to the boundary than the receptive field of the final layer; In~\semirealname, the digits are either close to the top boundary or the bottom boundary. Given that  the receptive field of Resnets is quite large, translation invariance would not hold in this case. This is further supported by recent work~\cite{kayhan2020translation}, which demonstrates that “CNNs can and will exploit the absolute spatial location by learning filters that respond exclusively to particular absolute locations by exploiting image boundary effects”.
We observe this phenomenon empirically in our \semirealnameone~ experiments as well. That is, while models trained on BlockMNIST-Top data (i.e., MNIST digit in top block) attain 100\% test accuracy on \semirealnameone~images, the accuracy of these models degrades to approximately 55\% (i.e., 5\% better than random chance) when evaluated on \texttt{BlockMNIST-Bottom} images, wherein the MNIST digit (signal) is placed in the bottom block.

\textbf{Choice of removal operator in \evalname~framework}.
Recall that in \evalname, the predictive power of a new model retrained on the unmasked dataset (i.e, data points after removal operation) is used to evaluate the fidelity of post-hoc explanation methods. 
Note that this approach employs retraining to account for and nullify distribution shifts induced by feature removal operators such as gaussian noise, zeros etc.
Since the same removal operation is applied to unmask every image (across classes), the choice of removal operator has no effect on our \evalname~results in~\Cref{sec:real}.
To verify this, we evaluated \evalname~on \cifar~with another removal operator in which pixels are masked/replaced by random gaussian noise (instead of zeros) and observed that the results do not change (i.e., same as in~\Cref{fig:real}). 

\textbf{Counterfactual changes vis-a-vis feature leakage}.
As evidenced in the \semirealname~experiments, input gradient attributions of standard models incorporate counterfactual changes in the null block. While this phenomenon seems natural and ``intuitive'' in hindsight, it can be misleading in the context of feature attributions. For example, consider the typical use case for feature attributions: to highlight regions within the given instance/image that are most relevant for model prediction. Now, in the~\semirealname~setting, if input gradients leak digit-like features into the null block, then the feature attributions in the null block can be easily (mis)interpreted as the non-discriminative null patch being highly relevant for model prediction.

\textbf{Comparison to results in~\citet{Kim2019BridgingAR}}.
\citet{Kim2019BridgingAR} use the \texttt{ROAR} framework to conjecture that adversarial training ``tilts'' input gradients to better align with the data manifold. 
First, in contrast to~\citet{Kim2019BridgingAR}, we thoroughly establish our \evalname~results across datasets/architectures/hyper-parameters, revealing a significantly larger gap between the attribution quality of standard and adversarially robust models.
Second, motivated by the boundary tilting hypothesis~\cite{tanay2016boundary}, \citet{Kim2019BridgingAR} use a two-dimensional synthetic dataset to empirically show that the decision boundary of robust models aligns better with the vector between the two class-conditional means. However, this empirical evidence might be misleading, as~\citet{ilyas2019adversarial} theoretically demonstrates that “this exact statement is not true beyond two dimensions” (pg. 15).
Furthermore, several recent works have also provided concrete evidence to support alternative hypotheses~\cite{ilyas2019adversarial,NEURIPS2020_6cfe0e61,shafahi2018adversarial,bubeck2019adversarial} for the existence of adversarial examples that counter the boundary tilting hypothesis that Kim et al. build upon.
This discrepancy in these results motivates the need for a multipronged approach, which we adopt to empirically identify the feature leakage hypothesis using BlockMNIST and theoretically verify the hypothesis in~\Cref{sec:synth}.

\textbf{Connections between adversarial robustness and data manifold}: In the recent past, there have been several results showing unexpected benefits of adversarially trained models beyond adversarial robustness such as visually perceptible input gradients~\cite{santurkar2019image} and feature representations that transfer better~\cite{salman2020adversarially}. One reason for this phenomenon widely considered in the literature~\cite{engstrom2019adversarial,shamir2021dimpled} is that the input data lies on a low dimensional manifold and unlike standard training, adversarial training encourages the decision boundary to lie on this manifold (i.e. alignment with data manifold). Our experiments and theoretical results on feature leakage suggest that this reasoning is indeed true for both the~\semirealname~and its simplified version presented in Section~\ref{sec:synth}. Furthermore, we believe that the simplified version of~\semirealname~in~eq. \eqref{eqn:synth} can be used as a tool to thoroughly investigate both the benefits and potential drawbacks of adversarially trained models.
\clearpage 

\textbf{Why focus on input gradient attributions?}.
As discussed in~\Cref{sec:intro}, several feature attributions such as guided backprop~\cite{springenberg2014striving} and integrated gradients~\cite{sundararajan2016gradients} that output visually sharper saliency maps fail basic sanity checks such as model randomization and label randomization~\cite{adebayo2018sanity,kindermans2019reliability,NEURIPS2019_a7471fdc}.
We focus on vanilla input gradient attributions for two key reasons: (i) vanilla input gradients pass both sanity checks mentioned above and (ii) the input gradient operation is the key building block of several feature attribution methods. 
Our experiments and theoretical analysis are specifically designed  to identify and verify feature leakage in input gradient attributions of standard models. 

\textbf{Comparing \texttt{ROAR} and \texttt{DiffROAR}}. The following questions below illustrate key differences between \texttt{ROAR}~\cite{Hooker2019ABF} and our work:
\begin{itemize}[leftmargin=*]
	\vspace{-5px}
	\item \textit{Does the framework verify assumption \premise?} In~\citet{Hooker2019ABF}, the \texttt{ROAR} framework essentially computes the top-$k$ predictive power only, which is not sufficient to test assumption \premise. In our paper, DiffROAR directly compares the top-$k$ and bottom-$k$ predictive power to test whether the given attribution method satisfies assumption \premise.
	\item \textit{Are the results in the paper conclusive?} Both, \texttt{ROAR} and \evalname, make a key assumption: models retrained on unmasked datasets learn the same features as the model trained on the original dataset. Although empirically supported~\cite{pmlr-v119-hacohen20a,pmlr-v44-li15convergent,NEURIPS2020_6cfe0e61}, this assumption makes it difficult to conclusively test assumption \premise. Therefore, we empirically (Section 5) and theoretically (Section 6) verify our \evalname~findings in settings wherein ground-truth features are known a priori.
	\item \textit{Does the work identify why standard input gradients violate \premise?} \citet{Hooker2019ABF} do not discuss why input gradients lack explanation fidelity. In our paper, we hypothesize feature leakage as the key reason for ineffectiveness of input gradients, and validate it with empirical as well as theoretical analysis on \semirealname-based data
\end{itemize}

\textbf{Limitations of \texttt{ROAR} and \evalname}. The major limitation of \texttt{ROAR} and \evalname~is the key assumption that models retrained on unmasked datasets learn the same features as the model trained on the original dataset. In the absence of ground-truth features, this assumption is empirically supported by findings that suggest that different runs of models sharing the same architecture learn similar features~\cite{pmlr-v119-hacohen20a,pmlr-v44-li15convergent,NEURIPS2020_6cfe0e61}. 
Another limitation is that \texttt{ROAR}-based frameworks are not useful in the following setting. Consider a redundant dataset where features are either all negative (in which case label $y=0$) or all positive (in which case label $y=1$). In such cases, no feature is more or less informative than any other, so no information can be gained by ranking or removing input coordinates/features.

\clearpage
\section{Experiments on real-world datasets using \evalname}
\label{appendix:real}
\vspace{-6px}



In this section, 
we first provide additional details about datasets, training, and performance of trained models vis-a-vis generalization and robustness.
We also present top-$k$ and bottom-$k$ predictive power of input gradient unmasking schemes obtained via standard and robust models.
Next, we show that our results on image classification benchmarks are robust to CNN architectures and SGD hyperparameters used during retraining.
Then, we use \evalname~to show that our results hold with input \emph{loss} gradients, but \emph{signed} input logit gradients do not satisfy assumption \premise~for standard \emph{or} robust models.
Finally, we discuss \evalname~results obtained {without} retraining and provide additional example images that are masked using input gradients of standard \& robust models.

\vspace{-5px}
\subsection{{Additional details about \evalname~experiments and trained models}}
\label{app:expt_deets}
We first provide additional details about standard and adversarial training, and describe the performance of trained models vis-a-vis generalization and robustness to $\ell_2$ \& $\ell_{\infty}$ perturbations.

Recall that we use~\evalname~to analyze input gradients of standard and adversarially robust two-hidden-layer MLPs on SVHN \& Fashion MNIST, Resnet18 on \imnet, and Resnet50 on~\cifar~in~\Cref{fig:real}.
In these experiments, we train models using stochastic gradient descent (SGD), with momentum $0.9$, batch size $256$, $\ell_2$ regularization $0.0005$ and initial learning rate $0.1$ that decays by a factor of $0.75$ every $20$ epochs;
We obtain $\ell_2$ and $\ell_{\infty}$ $\epsilon$-robust models with perturbation budget $\epsilon$ using PGD adversarial training~\cite{madry2018towards}.
In PGD adversarial training, we use learning rate $\sfrac{\epsilon}{4}$, $8$ steps of PGD and no random initialization in order to compute $\epsilon$-norm $\ell_2$ and $\ell_{\infty}$ perturbations.
In both cases, we use standard data augmentation and train models for at most $500$ epochs, stopping early if cross-entropy (standard or adversarial) loss on training data goes below $0.001$. 
Unless mentioned otherwise, we set the depth and width of MLPs trained on real datasets to be $2$ and $2\times$ the input dimension respectively.

\Cref{fig:app_perf} depicts standard test accuracy (i.e., when perturbation budget $\epsilon=0$) and $\epsilon$-robust test accuracy (for multiple values of $\epsilon$) of standard as well as $\ell_2$ and $\ell_{\infty}$ robust models trained on SVHN, Fashion MNIST, CIFAR-10 and ImageNet-10.
Note that to estimate $\epsilon$-robust test accuracy, we use PGD-based adversarial \emph{test} examples, computed using $2 \times$ the number of PGD steps used during training. 
As expected, we observe that (i) compared to standard models, adversarially trained MLPs and CNNs attain significantly better robust test accuracy, (ii) models trained with larger perturbation budget are more robust to larger-norm adversarial perturbations at test time, and (iii) standard test accuracies (when $\epsilon=0$) of adversarially trained models are worse than those of standard models.

\begin{figure*}[h]
	\centering
	\includegraphics[width=\linewidth]{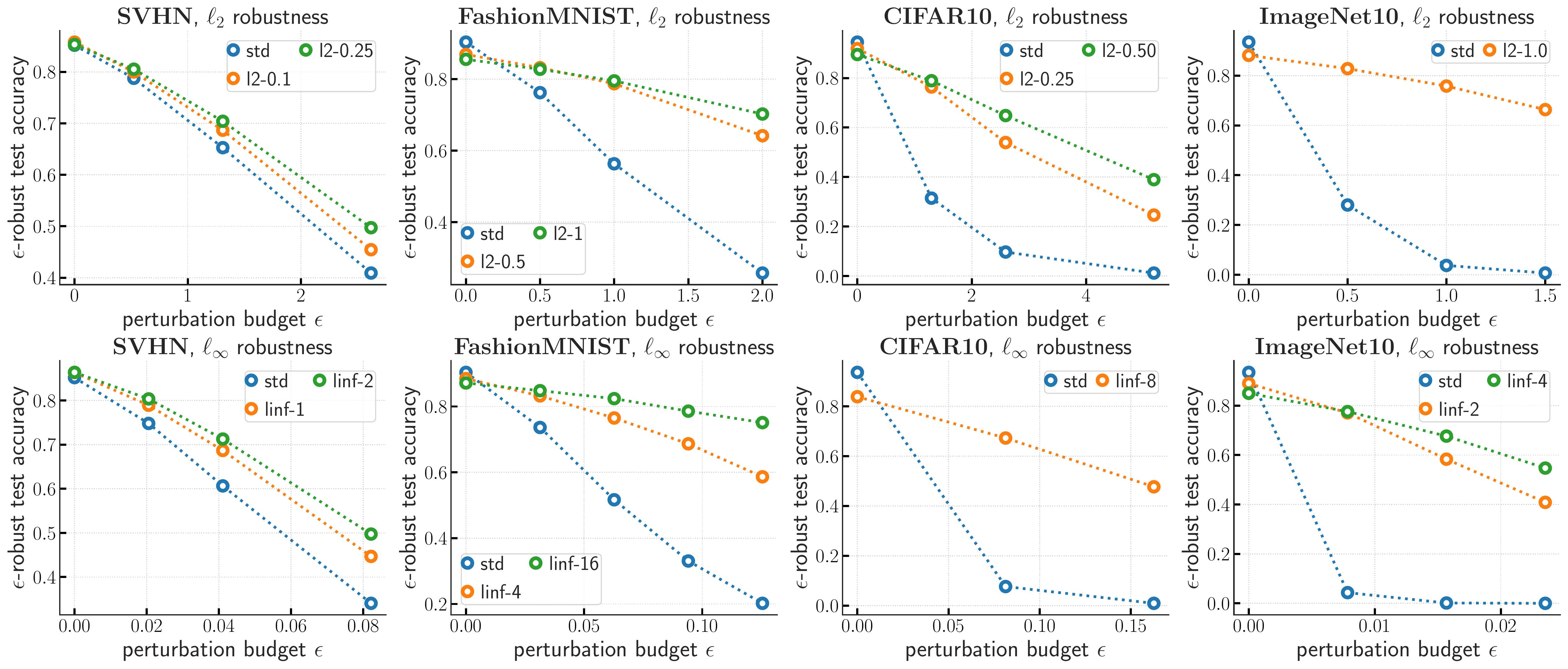}
	\vspace{-15px}
	\caption{
		\textbf{Standard and $\epsilon$-robust test accuracies} of MLPs trained on SVHN and Fashion MNIST, Resnet50 trained on CIFAR-10, and Resnet18 trained on ImageNet10. Details in~\Cref{app:expt_deets}.
	} 
 \label{fig:app_perf}
 \vspace{-5px}
\end{figure*}

\subsection{{Top-$k$ and bottom-$k$ predictive power of input gradient attributions}}
\label{app:expt_topk_bottomk}

Now, we describe the top-$k$ and bottom-$k$ predictive power curves for unmasking schemes of input gradients of standard and robust models.
Recall that top-$k$ predictive power simply estimates the test accuracy of models  that are retrained on datasets wherein only coordinates with top-$k$ (\%) of the coordinates are unmasked in every image.
The top and bottom rows in~\Cref{fig:app_tkbk} show how top-$k$ and bottom-$k$ predictive power of input gradient attributions of standard and robust models vary with unmasking fraction $k$ respectively. 
The subplots in~\Cref{fig:app_tkbk} show that (i) decreasing the unmasking fraction $k$ decreases top-$k$ and bottom-$k$ predictive power, and (ii) models retrained on attribution-masked datasets attain non-trivial unmasked test dataset accuracy even when a significant fraction of coordinates with the top-most and bottom-most attributions are masked. 

As described in~\Cref{sec:prelim}, for a given attribution scheme and unmasking fraction or level $k$, \evalname~(see equation~\eqref{eqn:diffroar}) is positive when the top-$k$ predictive power is greater than the bottom-$k$ predictive power.
The subplots in the first column indicate that standard models trained on Fashion MNIST do not satisfy assumption~\premise~because the top-$k$ and bottom-$k$ unmasking schemes are \emph{equally} ineffective at masking discriminative features.
Conversely, the difference between the top-$k$ and bottom-$k$ predictive power of input gradient attributions of robust models is significant. 
For example, in the second column, for the SVHN model adversarially trained with $\ell_{\infty}$ perturbations and budget $\epsilon=\sfrac{2}{255}$ (purple line), top-$k$ predictive power is roughly $40\%$ more than the bottom-$k$ predictive power when $k=5\%$. 
Furthermore, as shown in the third and fourth columns, the top-$k$ and bottom-$k$ curves of standard CNNs trained on CIFAR-10 and ImageNet-10 are ``inverted'', thereby explaining why  \evalname~is \emph{negative} when unmasking fraction is roughly less than $40\%$. 

\begin{figure*}[h]
	\centering
	\includegraphics[width=\linewidth]{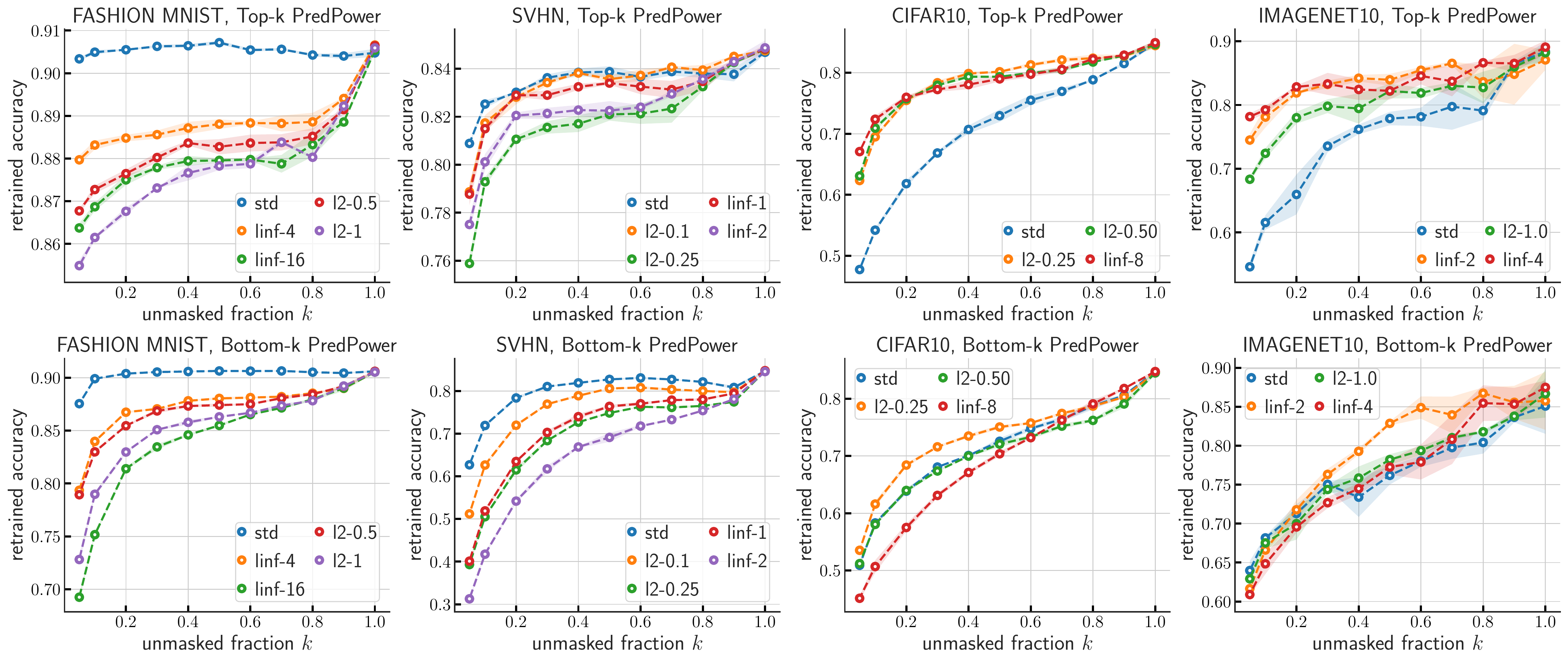}
	\vspace{-15px}
	\caption{
		\textbf{Predictive power of top-$k$ and bottom-$k$ input gradient unmasking schemes} vs. unmasking fraction, or level, $k$ for standard and adversarially robust models trained on $4$ image classification benchmarks. Please see~\Cref{app:expt_topk_bottomk} for details. 
	} 
 \label{fig:app_tkbk}
 \vspace{-5px}
\end{figure*}

\vspace{-5px}
\subsection{{Effect of model architecture on \evalname~results}}
\vspace{-5px}
\label{app:diffroar_arch}
Recall that in~\Cref{sec:real}, we used the~\evalname~metric to evaluate whether input gradient attributions of models trained on real-world datasets satisfy or violate assumption~\premise. 
For CNNs, we evaluated input gradient attributions of standard Resnet50 and Resnet18 models trained on CIFAR-10 and Imagenet-10 respectively. 
In this section, we show that our empirical findings based on these architectures extend to three other commonly-used and well-known CNN architectures: Densenet121, InceptionV3, and VGG11.

As shown in~\Cref{fig:app_arch}, the \evalname~results support key empirical findings made using input gradients of Resnet models in~\Cref{sec:real}: (i) standard models perform poorly, often no better or even worse than the random attribution baseline, and (ii) \evalname~curves of adversarially robust models are positive and significantly better than that of the standard model.
For example, for Densenet121, InceptionV3, and VGG11, when unmasking fraction $k=20\%$, standard training yields input gradient attributions that attain~\evalname~scores roughly $-5\%$, $2\%$ and $1\%$ respectively, whereas $\ell_{\infty}$ adversarial training with budget $\epsilon=\sfrac{6}{255}$ results in input gradients with \evalname~metric roughly $15\%$. 

\begin{figure*}[h]
	\centering
	\includegraphics[width=\linewidth]{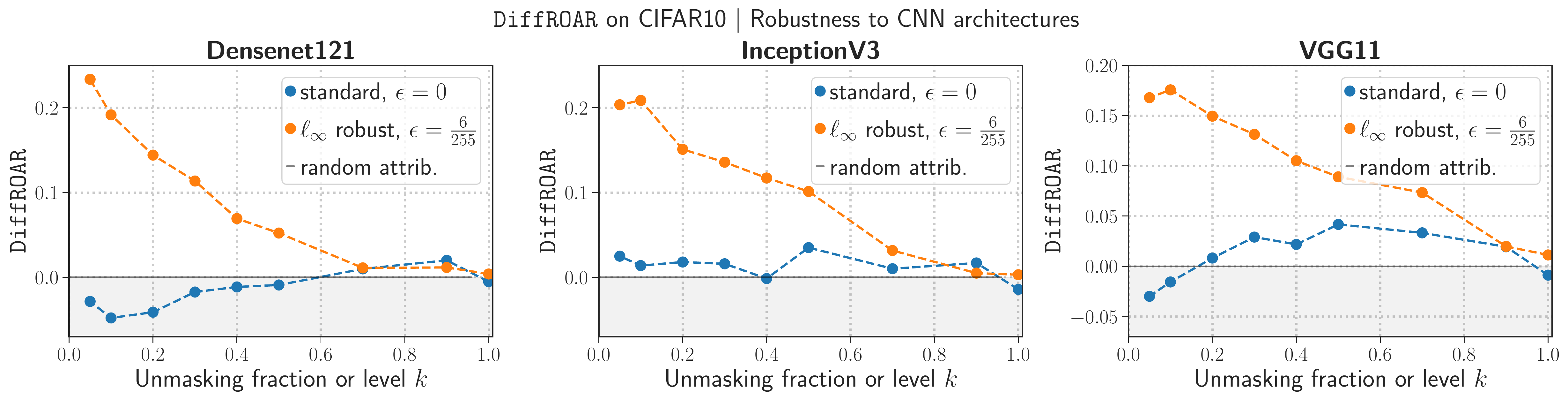}
	\vspace{-15px}
	\caption{
		\textbf{\evalname~results on input gradients of additional CNN architectures}.
		\evalname~curves for three well-known NN architectures---Densenet121, InceptionV3, and VGG11---indicate that empirical findings vis-a-vis input gradients of standard and robust models (\Cref{sec:real}) are robust to choice of CNN architecture. 
		Please see~\Cref{app:diffroar_arch} for details.	} 
 \label{fig:app_arch}
 \vspace{-5px}
\end{figure*}

\vspace{-5px}
\subsection{{Effect of SGD Hyperparameters on \evalname~results}}
\label{app:hyp}
\vspace{-5px}
\begin{figure*}[t]
	\centering
	\includegraphics[width=\linewidth]{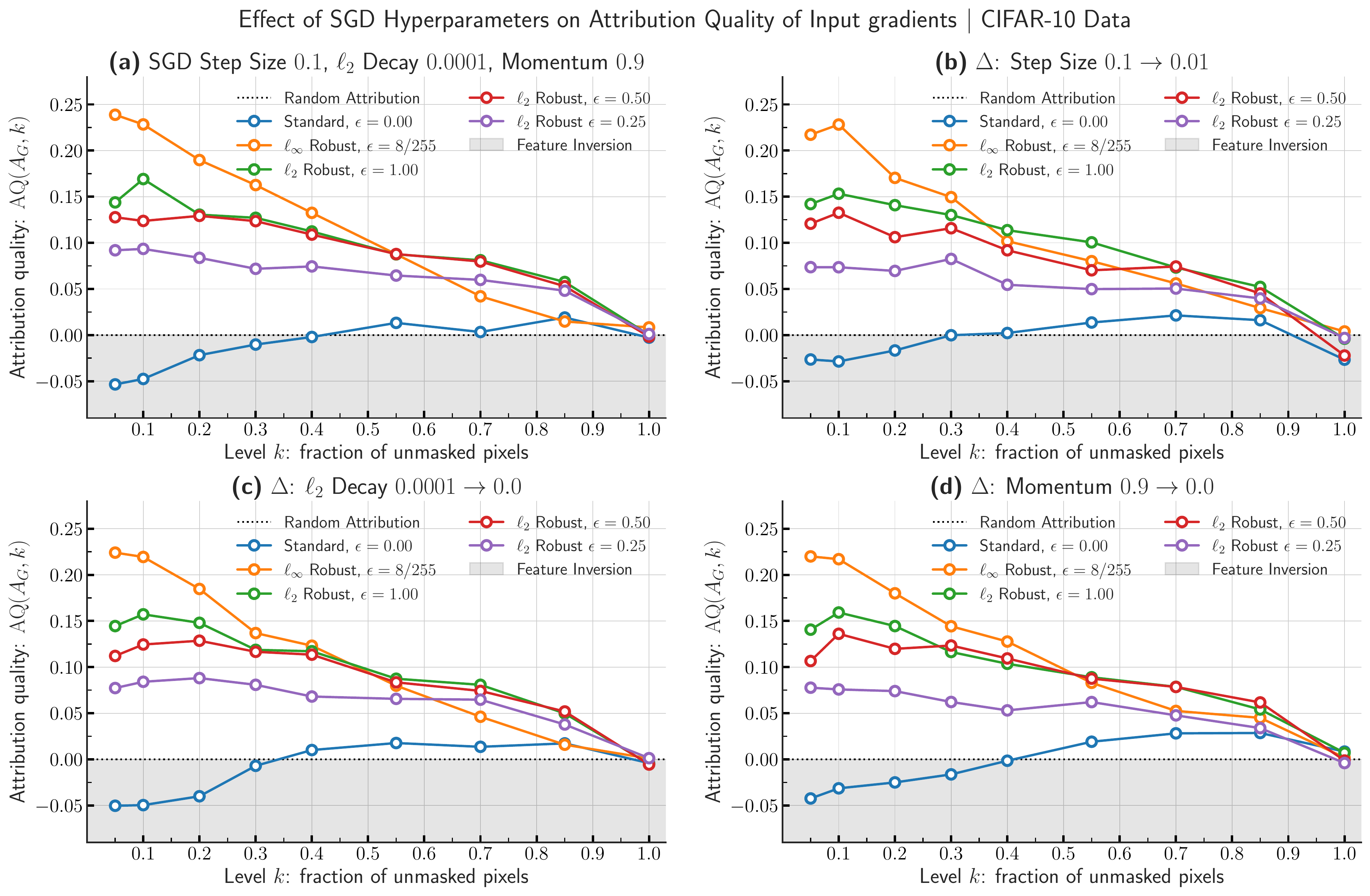}
	\vspace{-15px}
	\caption{
		\textbf{\evalname~robust to SGD hyperparameters in retraining}.
	\evalname~curves for input gradients of standard and robust models trained on \cifar~data show that our empirical findings presented in~\Cref{sec:real} are robust to SGD hyperparameters that are used in retraining. 
	Specifically, we show that our findings vis-a-vis \evalname~are not sensitive to changes in SGD hyperparameters such as learning rate, momentum, and weight decay that are used to retrain models on unmasked \cifar~data.
	For example, the subplots above show that across multiple SGD hyperparameter values, when the fraction of unmasked pixels $k < 30$-$40\%$, standard models violate \premise~whereas robust models satisfy \premise.
	See~\Cref{app:hyp} for details.
	} 
 \label{fig:app_sgd}
 \vspace{-15px}
\end{figure*}

In this section, we show that \evalname~results for input gradient attribu of standard and robust models are not sensitive to the choice of SGD hyperparameters used during retraining. 
In particular, we show that \evalname~curves on \cifar~are not sensitive to the learning rate, weight decay, or the momentum used in SGD to train models on top-$k$ or bottom-$k$ attribution-masked datasets.
The four subplots in~\Cref{fig:app_sgd} collectively show that decreasing learning rate from $0.1$ to $0.01$, weight decay from $0.0001$ to $0$, and momentum from $0.9$ to $0$ does not alter our findings: (i) input gradient attributions of standard models do not satisfy \premise~when unmasking fraction $k$ is roughly less than $30$-$40$\%; (ii) models that are robust to $\ell_2$ and $\ell{\infty}$ perturbations consistently satisfy \premise; (iii) increasing perturbation budget $\epsilon$ during PGD adversarial training increases \evalname~metric for most values of unmasking fraction $k$.
To summarize, our results based on the \evalname~evaluation framework are robust to SGD hyperparameters used to retrain models on top-$k$ and bottom-$k$ unmasked datasets.

\vspace{-5px}
\subsection{{Evaluating input \emph{loss} gradient attributions using \evalname}}
\vspace{-5px}
\begin{figure*}[t]
	\centering
	\includegraphics[width=\linewidth]{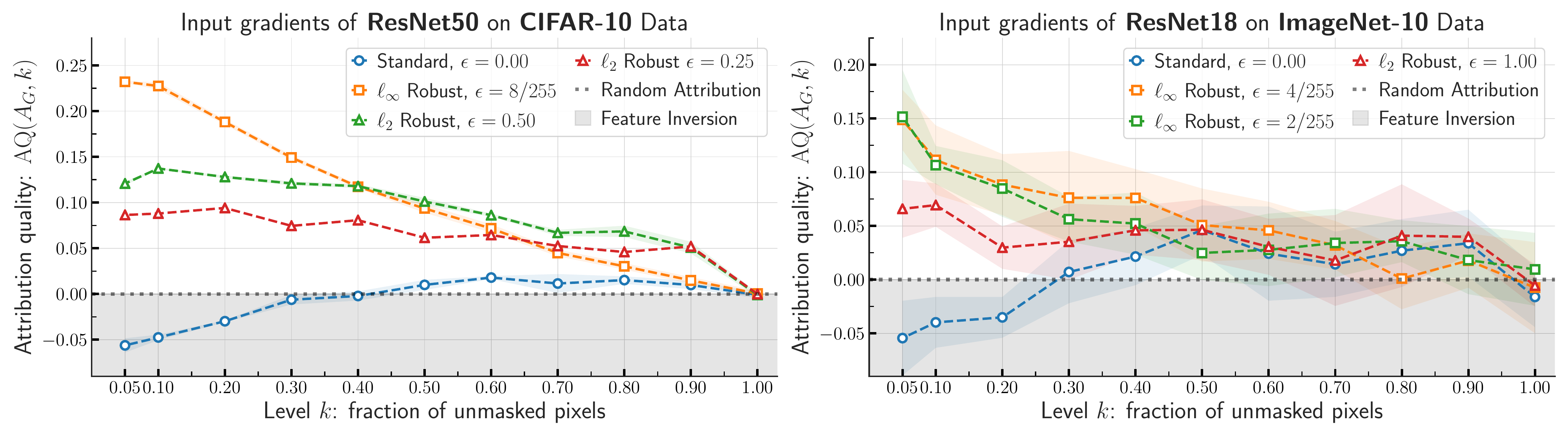}
	\vspace{-15px}
	\caption{
		\textbf{\evalname~results for input loss gradient attributions}.
	\evalname~plots for input \emph{loss} gradient attributions of standard and adversarially Resnet50 on \cifar~and Resnet18 on~\imnet.
	In both subplots, standard models violate \premise~when the fraction of unmasked pixels $k < 30\%$.
	That is, input coordinates that have the largest gradient magnitude are not as important performance-wise as the coordinates with smallest gradient magnitude.
	Conversely, $\{\ell_2,\ell_{\infty}\}$-adversarially trained models satisfy \premise, as the \evalname~metric is positive for all $k < 100\%$.
	Similar to our results with input logit gradients, we observe that increasing the perturbation budget $\epsilon$ during adversarial training amplifies the magnitude of \evalname~for every $k$ across all four image classification benchmarks.
	}
 \label{fig:app_loss}
 \vspace{-5px}
\end{figure*}

Recall that our experiments in~\Cref{sec:real} evaluate whether input gradients taken w.r.t. the logit of the predicted label satisfy or violate assumption~\premise~on image classification benchmarks.
In this section, we show that our empirical findings generalize to input \emph{loss} gradients---input gradients w.r.t loss (e.g., cross-entropy)---of standard and robust models evaluated on image classification benchmarks.
Specifically, we apply \evalname~ to input \emph{loss} gradients of standard and robust ResNet models trained on \cifar~and~\imnet.

\Cref{fig:app_loss} illustrates \evalname~curves for input \emph{loss} gradient attributions on \cifar~and~\imnet~data.
In both cases, we observe that (i) input loss gradient attributions of robust models, unlike those of standard models, satisfy \premise and (ii) PGD adversarial training  with larger perturbation budget $\epsilon$ increases the \evalname~metric in a consistent manner. Recall that the magnitude in \evalname~quantifies the extent to which the attribution order separates discriminative and task-relevant features from features that are unimportant for model prediction; see~\Cref{sec:prelim} for more information about \evalname.

\vspace{-5px} 
\subsection{{Evaluating \emph{signed} input gradient attributions using \evalname}}
\label{app:signed}
\vspace{-5px}
In addition to input loss gradient magnitude attributions and input logit gradient magnitude attributions, our results vis-a-vis \evalname~evaluation on image classification benchmarks extend to \emph{signed} input logit gradients as well.
In signed input gradient attributions, input coordinates are ranked based on $\text{sgn}{x_i} \cdot g_i$ where $\text{sgn}(x_i)$ is the sign of input coordinate $x_i$ and $g_i$ is the signed input gradient value for input coordinate $x_i$. 
\vspace{-10px}

\begin{figure*}[h]
	\centering
	\includegraphics[width=0.95\linewidth]{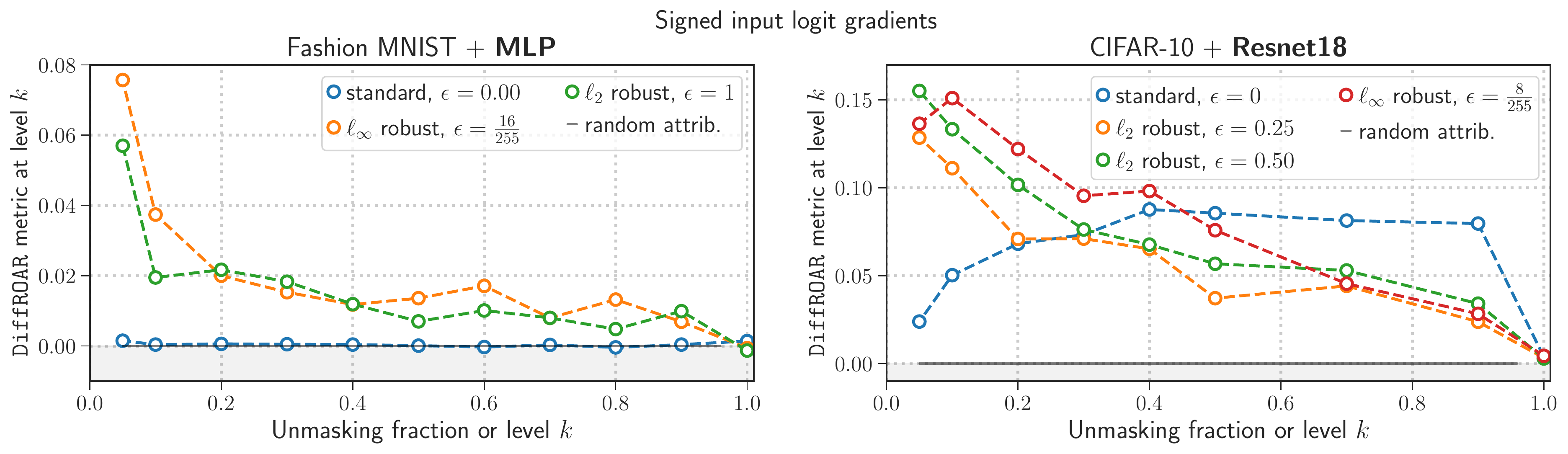}
	\vspace{-6px}
	\caption{
		\textbf{\evalname~results for signed input logit gradients}.
		\evalname~results for attributions based on signed input gradients of standard and robust MLPs \& CNNs trained on Fashion MNIST \& CIFAR-10. See~\Cref{app:signed} for details.
	} 
 \label{fig:app_signed}
 \vspace{-10px}
\end{figure*}

\Cref{fig:app_signed} shows \evalname~curves for attributions based on \emph{signed} input gradients taken with respect to the logit of the predicted label. 
The left and right subplot evaluate \evalname~for standard and robust (i) MLP trained on Fashion MNIST and (ii) Resnet18 models trained on CIFAR-10. 
Consistent with our findings in~\Cref{sec:real}, while standard MLPs trained on Fashion MNIST fare no better than random attributions, signed input gradients of robust MLPs attain positive \evalname~scores for all $k < 100\%$ and perform considerably better than gradients of standard MLPs. 
Similarly, based on the \evalname~metric, when $k<50\%$, while signed input gradients of standard Resnet18 models perform better than absolute logit and loss gradients, signed input gradients of robust Resnet18 models continue to fare better than standard models.

\subsection{{The role of retraining in \evalname~evaluation}}
Figure~\ref{fig:app_no_retraining} shows the results on \evalname~without retraining on the masked datasets. As we can see from the figures, the trends are not consistent across model architectures and datasets, possibly due to varying levels of distribution shift. For this reason, we employ \evalname~with retraining as described in Section~\ref{sec:prelim}.

\begin{figure*}[h]
	\centering
	\includegraphics[width=\linewidth]{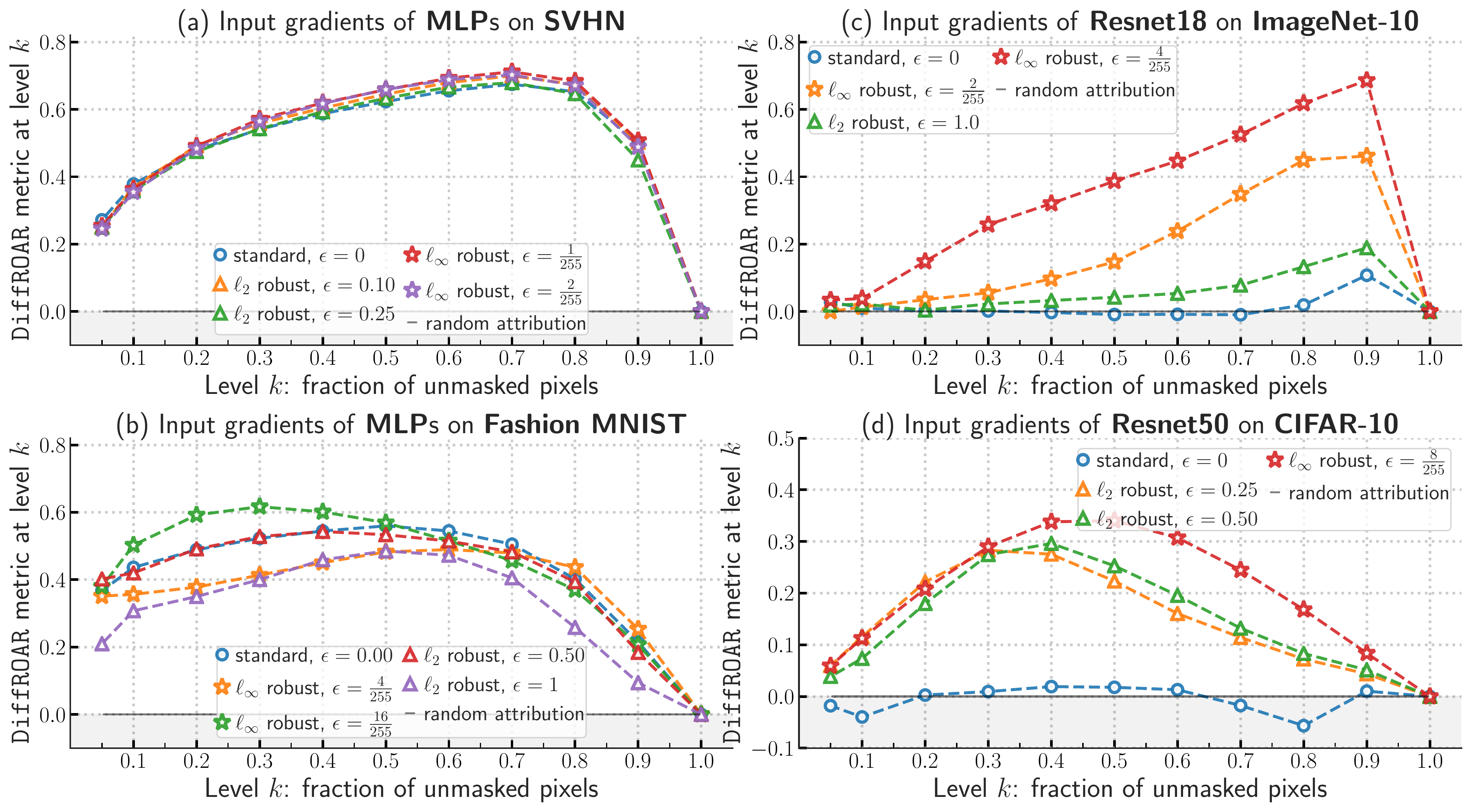}
	\vspace{-15px}
	\caption{
		\textbf{\evalname~results without retraining}. While we observe that standard models violate \premise~while adversarially trained models satisfy \premise~for the Resnet models, we see that both standard and adversarially trained models satisfy \premise~for MLP models, showing that this evaluation methodology does not yield consistent results across model architectures/datasets. Further, the \evalname~metric may be unrealiable for small unmasking fractions since this incurs heavy distribution shift. Consequently, we employ \evalname~after retraining on the new unmasked data.
	} 
 \label{fig:app_no_retraining}
 \vspace{-10px}
\end{figure*}

\subsection{{Imagenet-10 images unmasked using input gradients attributions of Resnet18 models}}
\label{app:viz}
\vspace{-5px}
\begin{figure*}[b]
	\centering
	\includegraphics[width=\linewidth]{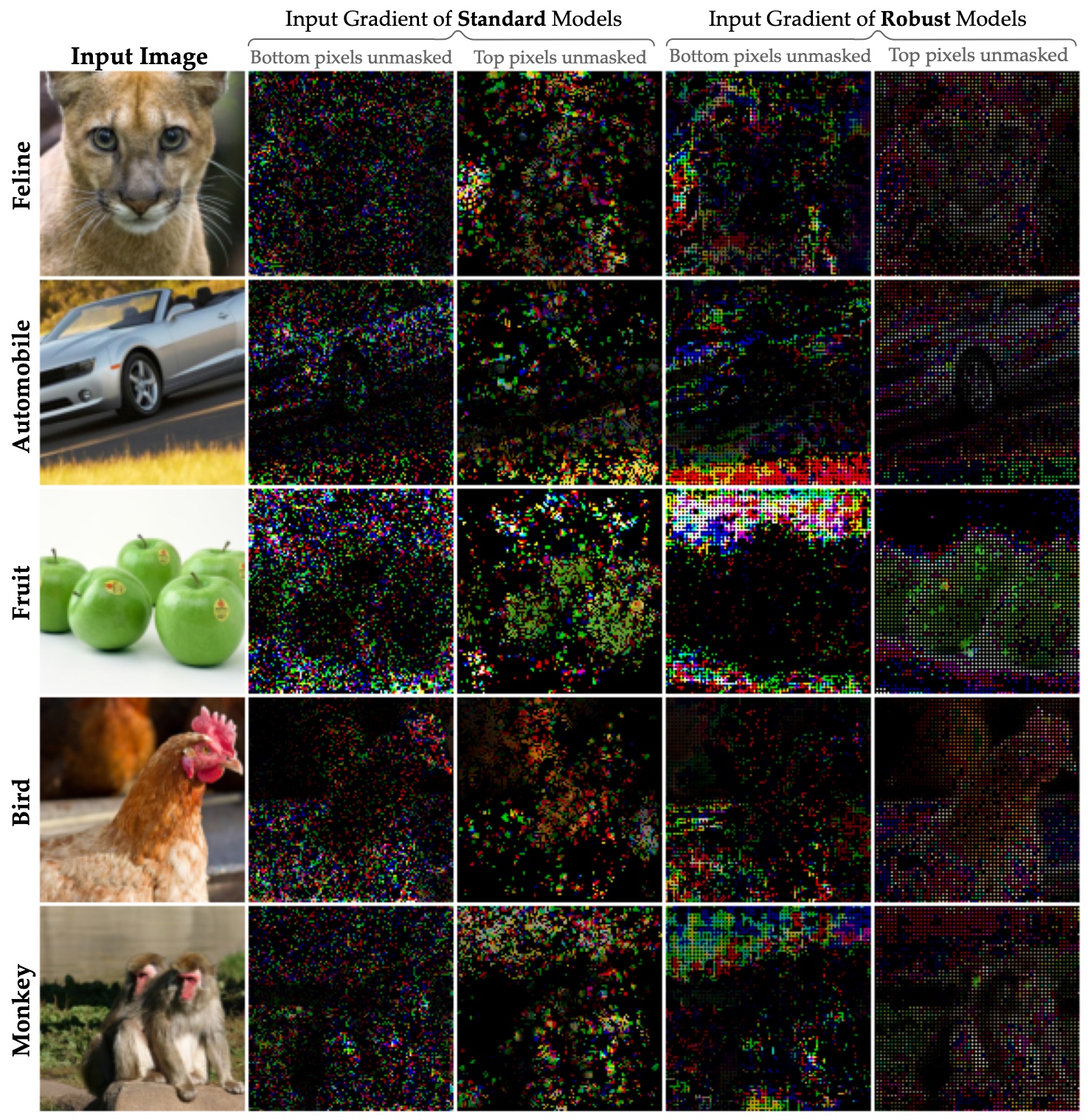}
	\caption{
		\textbf{ImageNet10 images unmasked using input gradient attributions}.
	Visualizing ImageNet-10 images that are unmasked using unmasking fraction, or level, $k=15\%$ using input gradient attributions of standard and $\ell_{\infty}$-robust Resnet18 models.
	Top-$k$ unmasked images (i.e., images in which \emph{only} ``top'' gradient attributions are unmasked) and bottom-$k$ unmasked images, attained via input gradients of standard models, share visual commonalities, suggestive of poor attribution quality. 
	Unlike bottom-$k$ unmasked images, images unmasked using top-$k$ attributions of robust models' input gradients highlight salient aspects  of images.
	See~\Cref{app:viz} for details.
	} 
 \label{fig:app_viz}
 \vspace{-5px}
\end{figure*}

Recall that in~\Cref{sec:real}, we showed that unlike input gradients of standard models, robust models consistently satisfy assumption~\premise. 
That is, input gradients of robust models highlight discriminative features, whereas input gradients of standard models tend to highlight non-discriminative features and suppress discriminative task-relevant features.
In this section, we qualitatively substantiate these findings by visualizing~\imnet~images that are unmasked using top-$k$ and bottom-$k$ input gradient attributions of standard and robust Resnet18 models. 
\emph{Please note that the following visual assessments are only meant to \emph{qualitatively} support findings made in~\Cref{sec:real} using the evaluation framework described in~\Cref{sec:prelim}.}
As discussed in~\Cref{sec:prelim}, if input gradients attain high-magnitude \evalname~score, images unmasked using top-$k$ attributions should highlight discriminative features, whereas images unmasked using bottom-$k$ should highlight non-discriminative features.

We make two observations using~\Cref{fig:app_viz} that qualitatively support our empirical findings in~\Cref{sec:real}.
First, we observe that images unmasked using top-$k$ gradient attributions of robust models tend to highlight salient aspects of images (e.g., shape of fruit or face of monkey in~\Cref{fig:app_viz}), whereas bottom-$k$ attributions often mask salient aspects of images either completely or partially.
Second, images unmasked using top-$k$ and bottom-$k$ attributions using input gradients of standard models exhibit visual commonalities, supporting the fact that for standard models, \evalname~is close to $0$ for multiple values of $k$.

\clearpage

\clearpage
\section{Additional experiments on feature leakage and \semirealname~data}
\label{appendix-semireal}


In this section, we first provide additional evidence that supports the feature leakage hypothesis in the setting used in~\Cref{sec:semireal}: \semirealname~data with \texttt{MNIST} digits $0$ and $1$ corresponding to the signal block in class $0$ and class $1$ respectively.
Then, we show that our results vis-a-vis feature leakage and \semirealname~are robust to the choice of \texttt{MNIST} digits used in the signal block as well as the number of classes in the \semirealname~classification task.
Finally, we end with a brief description of experiments that we conducted in order to test another hypothesized cause to understand why input gradients of standard models tend to violate~\premise.

\subsection{Additional analysis to demonstrate feature leakage in~\semirealname~data}
\label{app:leakage}
In this section, we provide (i) additional examples of~\semirealname~images and inputs gradients of standard and robust models, (ii) additional examples of \semirealnameone~images and input gradients, and (iii) describe a \emph{proxy} metric to measure feature leakage in \semirealname-based data.

\Cref{fig:app_blockmnist_01} shows $40$ \semirealname~images in the first row and their corresponding input gradients for standard and robust MLPs and Resnet18 models in the subsequent rows.
We observe that input gradient attributions of standard MLP and Resnet18 models consistently highlight the signal block \emph{as well as} the non-discriminative null block for all images. 
On the other hand, input gradient attributions of $\ell_2$ robust MLP and Resnet18 models exclusively highlight \texttt{MNIST} digits in the signal block and clearly suppress the square patch in the null block. 
These results further substantiate our results in~\Cref{fig:real} by showing that unlike standard models, adversarially robust models satisfy \premise~on \semirealname~data. 
\Cref{fig:app_blockmnist_top_01_topk} provides $20$ \semirealnameone~images in the first row and the corresponding input gradients of standard MLP and Resnet models in the subsequent rows. 
As shown in~\Cref{fig:app_blockmnist_01_topk}, in this setting, in contrast to results on \semirealname, input gradients of \emph{standard} Resnet18 and MLP models trained on~\semirealnameone~ satisfy assumption \premise. 

We further substantiate these findings using a \emph{proxy} metric to quantitatively measure feature leakage in \semirealname-based datasets.
As discussed in~\Cref{sec:semireal}, in the \semirealname~setting, we can restate assumption \premise~ as follows: \emph{Do input gradient attributions highlight the signal block over the null block?}
We measure the extent to which input gradients of a given trained model satisfies assumption \premise~by evaluating the fraction of top-$k$ attributions that are placed in the null block. 
In~\Cref{fig:app_blockmnist_01_topk}, we show that the fraction of top-$k$ attributions in the null block, when averaged over all images in the test dataset, is significantly greater for standard MLPs \& CNNs than for robust MLP \& CNNs. 
In~\Cref{fig:app_blockmnist_top_01_topk}, we show that input gradient attributions of standard models trained on~\semirealnameone~place significantly fewer attributions in the null block, compared to attributions of standard models trained on~\semirealname.
In both cases, the proxy metric further validates our findings vis-a-vis input gradients of standard \& robust models and feature leakage.

\begin{figure*}[b]
	\centering
	\includegraphics[width=\linewidth]{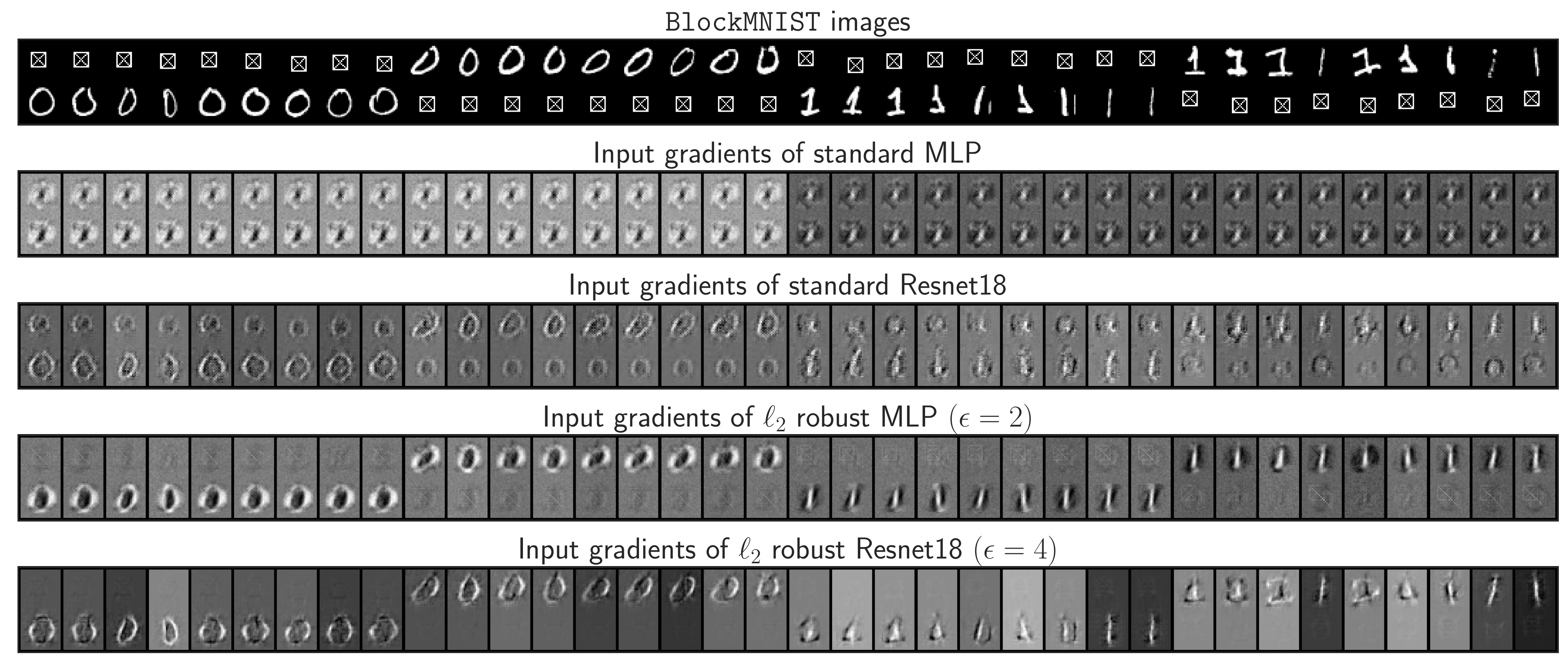}
	\vspace{-15px}
	\caption{
			\textbf{BlockMNIST 0 vs. 1}.
		$40$ \semirealname~(\texttt{MNIST} $0$ vs. $1$) images and their corresponding input gradients. 
		Recall that every image consists of a \emph{signal} and \emph{null} block, each randomly placed at the \emph{top} or \emph{bottom}. The \emph{signal} block, containing the \texttt{MNIST} digit $0$ or $1$, determines the image class, $0$ or $1$. 
		The \emph{null} block, containing the square patch, does not encode any information of the image class.
		The second, third, and fourth rows show input gradients of standard Resnet18, standard MLP, $\ell_2$ robust Resnet18 $(\epsilon=2)$ and $\ell_2$ robust MLP $(\epsilon=4)$ respectively.
		The plots clearly show that input gradients of standard \semirealname~models incorrectly highlight \emph{the non-discriminative null block} as well, thereby violating \premise.
		In contrast, input gradients of robust models highlight the signal block, suppress the null block, and satisfy \premise. 
		See~\Cref{app:leakage} for details.
	} 
 \label{fig:app_blockmnist_01}
 \vspace{-5px}
\end{figure*}

\begin{figure*}[b]
	\centering
	\includegraphics[width=\linewidth]{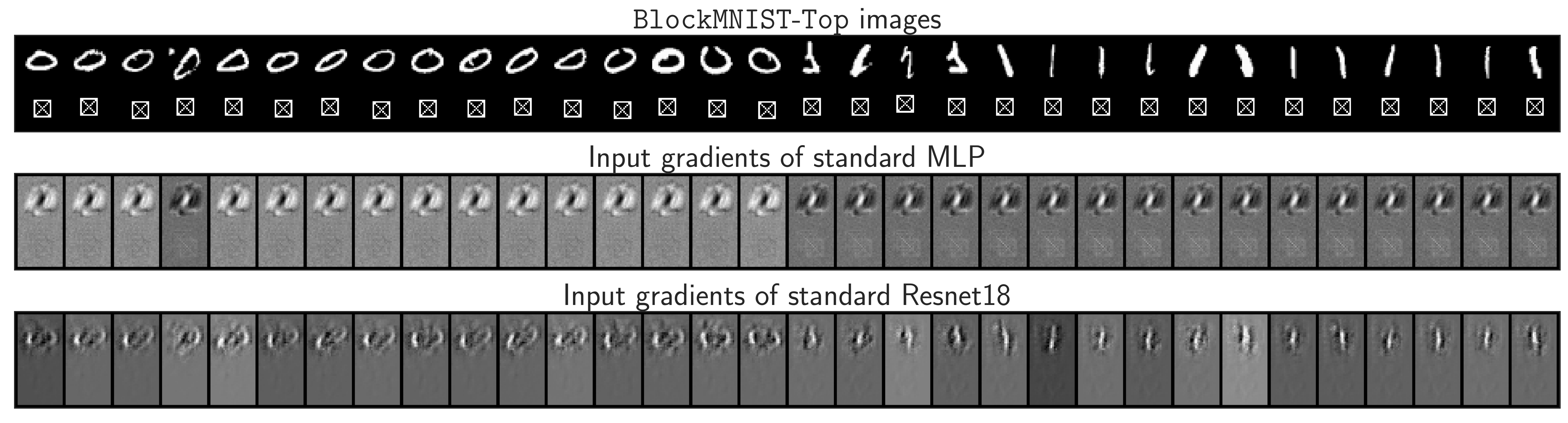}
	\vspace{-15px}
	\caption{
		\textbf{BlockMNIST-Top 0 vs. 1}.
		$20$ \semirealnameone~(\texttt{MNIST} $0$ vs. $1$) images and input gradients of standard MLP and Resnet18 models.
		As shown in the first row, 	the signal \& null blocks are fixed at the top \& bottom respectively in \semirealnameone~images.
		In contrast to results on \semirealname~in~\cref{fig:semireal_robust}, input gradients of standard models trained on~\semirealnameone~highlight the signal block, suppress the null block, and satisfy \premise. 
		Please see~\Cref{app:leakage} for details.
		} 
 \label{fig:app_blockmnist_top_01}
 \vspace{-5px}
\end{figure*}

\begin{figure*}[t]
	\centering
	\includegraphics[width=\linewidth]{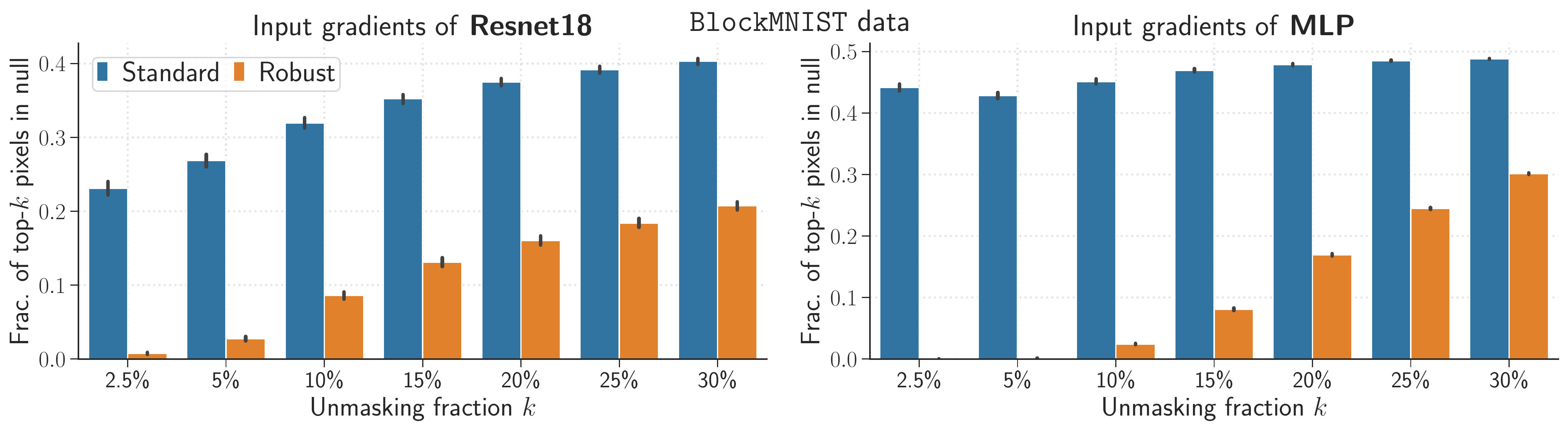}
	\vspace{-15px}
	\caption{
		\textbf{Proxy metric to compare input gradients of standard and robust models trained on \semirealname~($0$ vs. $1$) data}. 
The proxy metric measures the fraction of top-$k$ attributions that are placed in the null block of images in the test dataset. 
The left and right subplots evaluate this metric on input gradient attributions of standard and robust Resnet18 models and MLPs respectively. 
Compared to input gradients of standard models, adversarially trained models place significantly fewer top-$k$ attributions in the null block for multiple values of unmasking fraction $k$. Details in~\Cref{app:leakage}.
	} 
 \label{fig:app_blockmnist_01_topk}
 \vspace{-5px}
\end{figure*}

\begin{figure*}[t]
	\centering
	\includegraphics[width=\linewidth]{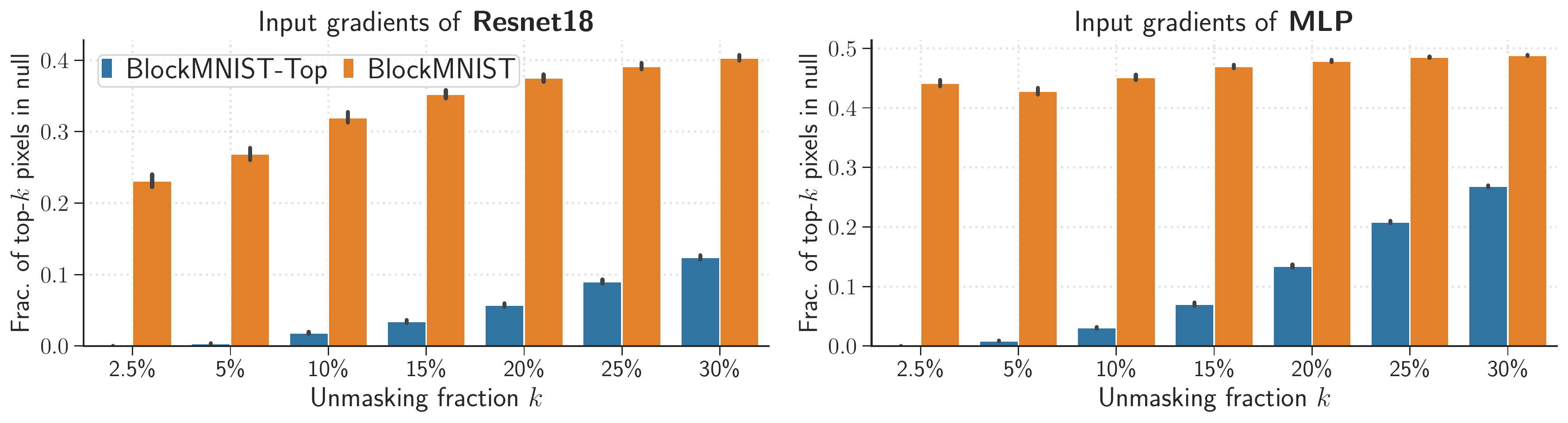}
	\vspace{-15px}
	\caption{
		\textbf{Proxy metric to compare input gradients of standard models trained on \semirealname~and \semirealnameone~($0$ vs. $1$) data}. The proxy metric measures the fraction of top-$k$ attributions that are placed in the null block of images. The left and right subplots evaluate this metric on input gradient attributions of standard Resnet18 models and MLPs trained on \semirealname~and \semirealnameone~data respectively.  Compared to input gradients of models trained on~\semirealname, standard models trained on~\semirealnameone~place significantly fewer top-$k$ attributions in the null block for multiple values of unmasking fraction $k$. Details in~\Cref{app:leakage}.
	} 
 \label{fig:app_blockmnist_top_01_topk}
 \vspace{-5px}
\end{figure*}

\subsection{Effect of choice and number of classes in~\semirealname~data}
\label{app:bm_rob}
In this section, we show that our analysis on~\semirealname-based datasets in~\Cref{sec:semireal} is robust to the choice and number of classes in~\semirealname~data.
In particular, we reproduce our empirical findings vis-a-vis feature leakage and input gradient attributions of standard vs. robust models on three additional~\semirealname-based tasks.
In~\Cref{fig:app_blockmnist_24} and~\Cref{fig:app_blockmnist_top_24}, we evaluate input gradients of standard and robust models trained on~\semirealname~and \semirealnameone~data, wherein the \texttt{MNIST} digits in class $0$ and class $1$ correspond to digits $2$ and $4$ (in the signal block) respectively. 
Similarly, in~\Cref{fig:app_blockmnist_73} and~\Cref{fig:app_blockmnist_top_73}, we reproduce our empirical findings from~\Cref{sec:semireal} on~\semirealname~and \semirealnameone~data in which the \texttt{MNIST} digits in class $0$ and class $1$ correspond to digits $3$ and $7$ (in the signal block) respectively. 
In~\Cref{fig:app_blockmnist_multiclass} and~\Cref{fig:app_blockmnist_top_multiclass}, we show that (i) input gradients of standard models violate assumption~\premise due to feature leakage and (ii) adversarial training mitigates feature leakage on $10$-class \semirealname~and \semirealnameone~data, wherein each class $i \{0,\ldots,9\}$ corresponds to \texttt{MNIST} digit $i$ in the signal block.

\begin{figure*}[t]
	\centering
	\includegraphics[width=\linewidth]{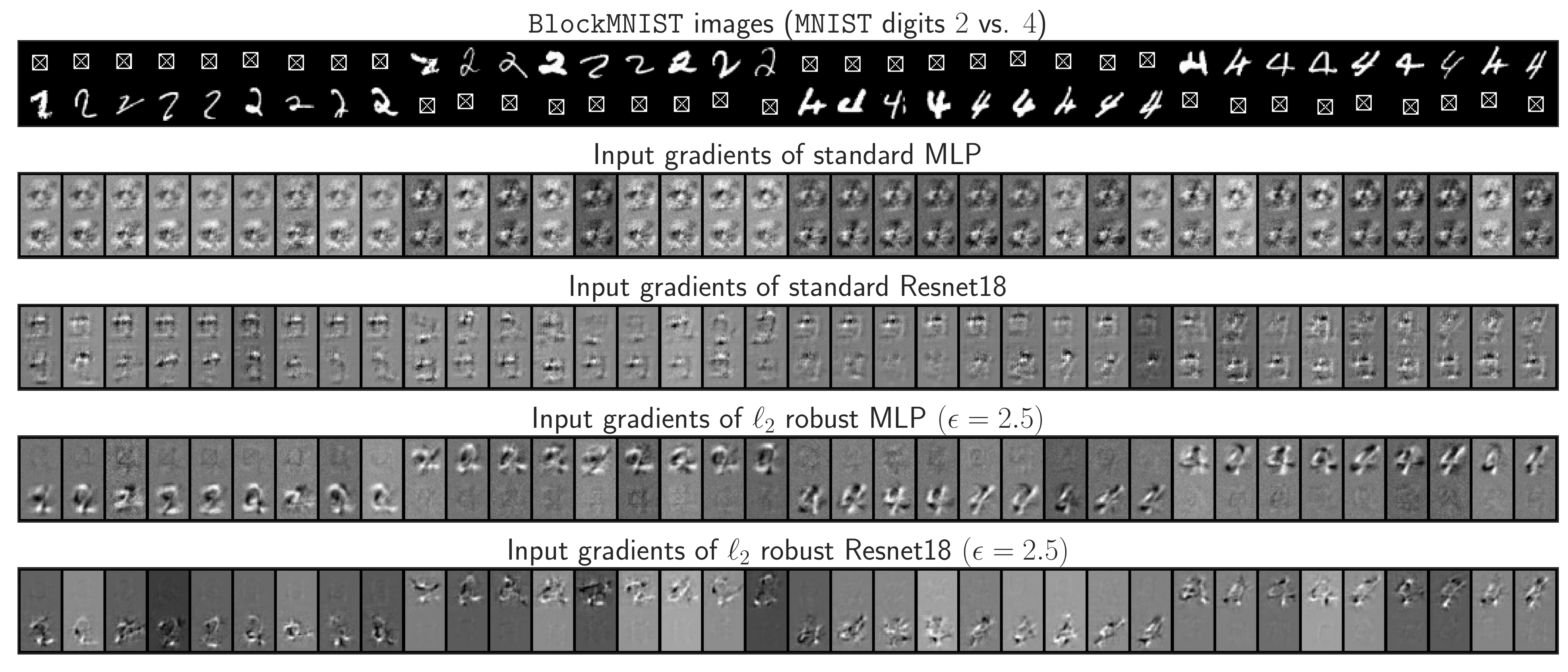}
	\vspace{-15px}
	\caption{
			\textbf{BlockMNIST 2 vs. 4}.
		$40$ \semirealname~(\texttt{MNIST} $2$ vs. $4$) images and their corresponding input gradients. 
		The \emph{signal} block, containing the \texttt{MNIST} digit $2$ or $4$, determines the image class, $0$ or $1$. 
		The second, third, and fourth rows show input gradients of standard Resnet18, standard MLP, $\ell_2$ robust Resnet18 $(\epsilon=2.5)$ and $\ell_2$ robust MLP $(\epsilon=2.5)$ respectively.
		The plots clearly show that input gradients of standard \semirealname~models incorrectly highlight \emph{the non-discriminative null block} as well, thereby violating \premise.
		In contrast, input gradients of robust models highlight the signal block, suppress the null block, and satisfy \premise. 
		See~\Cref{app:leakage} for details.
	} 
 \label{fig:app_blockmnist_24}
 \vspace{-5px}
\end{figure*}

\begin{figure*}[t]
	\centering
	\includegraphics[width=\linewidth]{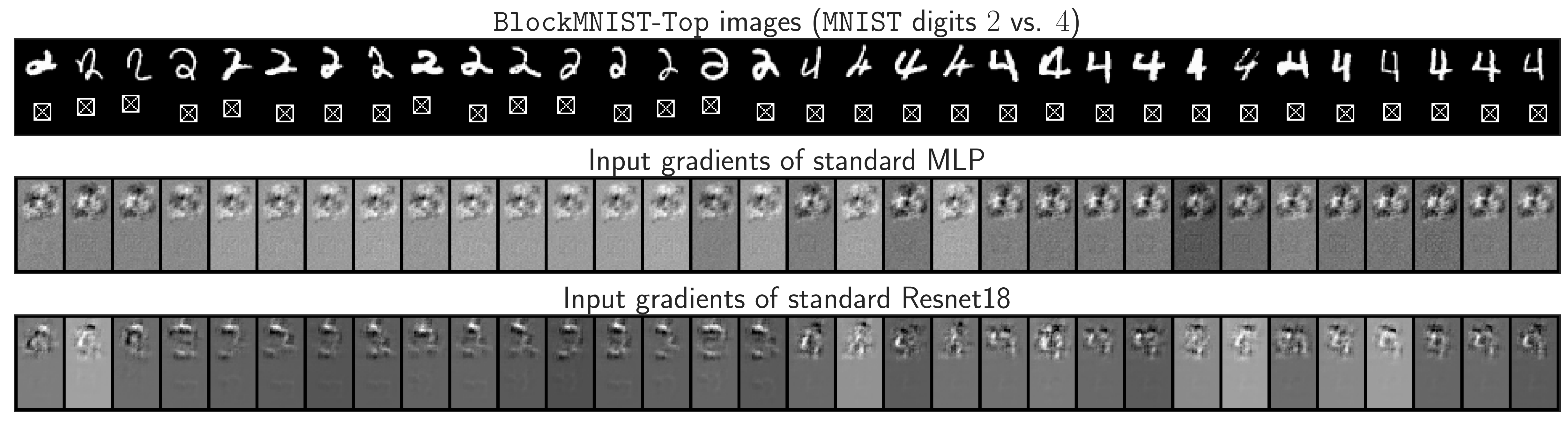}
	\vspace{-15px}
	\caption{
		\textbf{BlockMNIST-Top 2 vs. 4}.
		$20$ \semirealnameone~(\texttt{MNIST} $2$ vs. $4$) images and corresponding input gradients of standard MLP and Resnet18 models.
		As shown in the first row, 	the signal \& null blocks are fixed at the top \& bottom respectively in \semirealnameone~images.
		In contrast to results on \semirealname~in~\cref{fig:semireal_robust}, input gradients of standard models trained on~\semirealnameone~highlight the signal block, suppress the null block, and satisfy \premise. 
		Please see~\Cref{app:leakage} for details.
	} 
 \label{fig:app_blockmnist_top_24}
 \vspace{-5px}
\end{figure*}

\begin{figure*}[t]
	\centering
	\includegraphics[width=\linewidth]{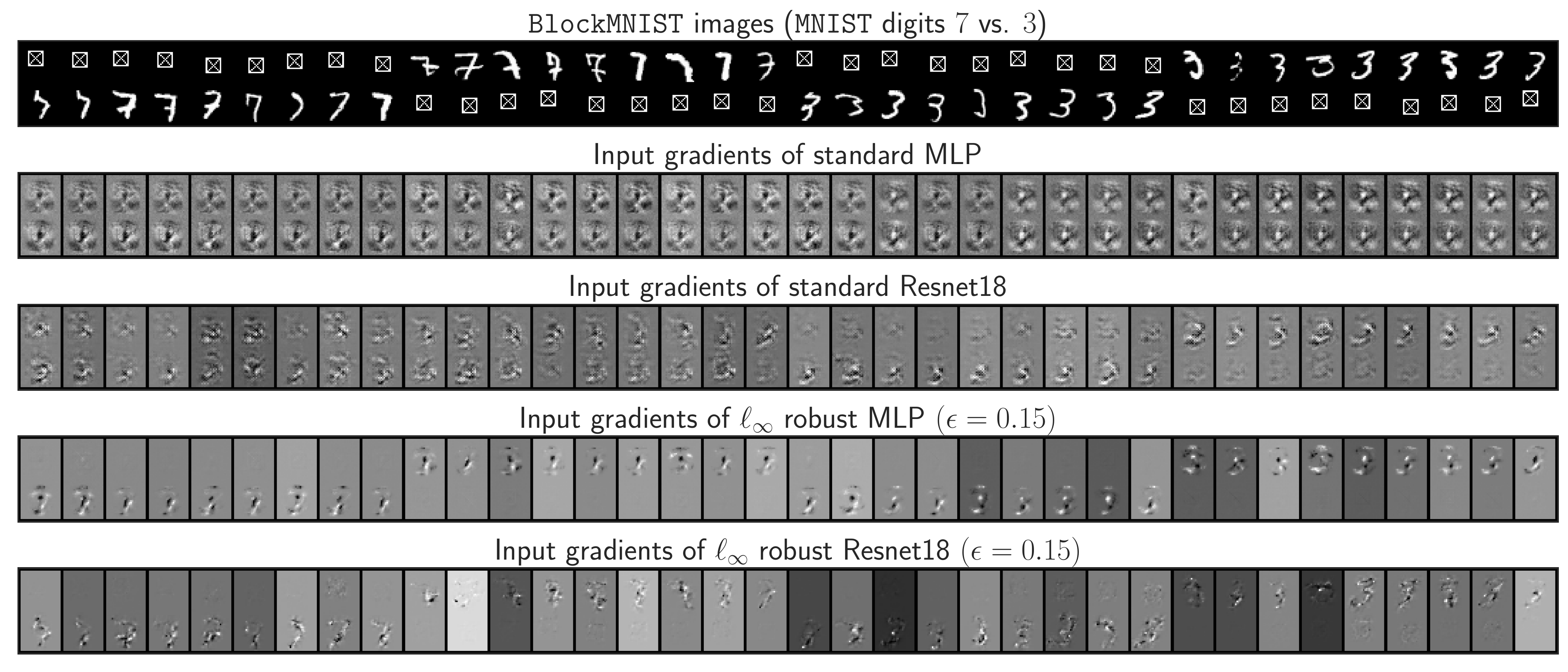}
	\vspace{-15px}
	\caption{
			\textbf{BlockMNIST 3 vs. 7}.
		$40$ \semirealname~(\texttt{MNIST} $3$ vs. $7$) images and their corresponding input gradients. 
		The \emph{signal} block, containing the \texttt{MNIST} digit $3$ or $7$, determines the image class, $0$ or $1$. 
		The second, third, and fourth rows show input gradients of standard Resnet18, standard MLP, $\ell_{\infty}$ robust Resnet18 $(\epsilon=0.15)$ and $\ell_{\infty}$ robust MLP $(\epsilon=0.15)$ respectively.
		The plots clearly show that input gradients of standard \semirealname~models incorrectly highlight \emph{the non-discriminative null block} as well, thereby violating \premise.
		In contrast, input gradients of robust models highlight the signal block, suppress the null block, and satisfy \premise. 
		See~\Cref{app:leakage} for details.
	} 
 \label{fig:app_blockmnist_73}
 \vspace{-5px}
\end{figure*}

\begin{figure*}[t]
	\centering
	\includegraphics[width=\linewidth]{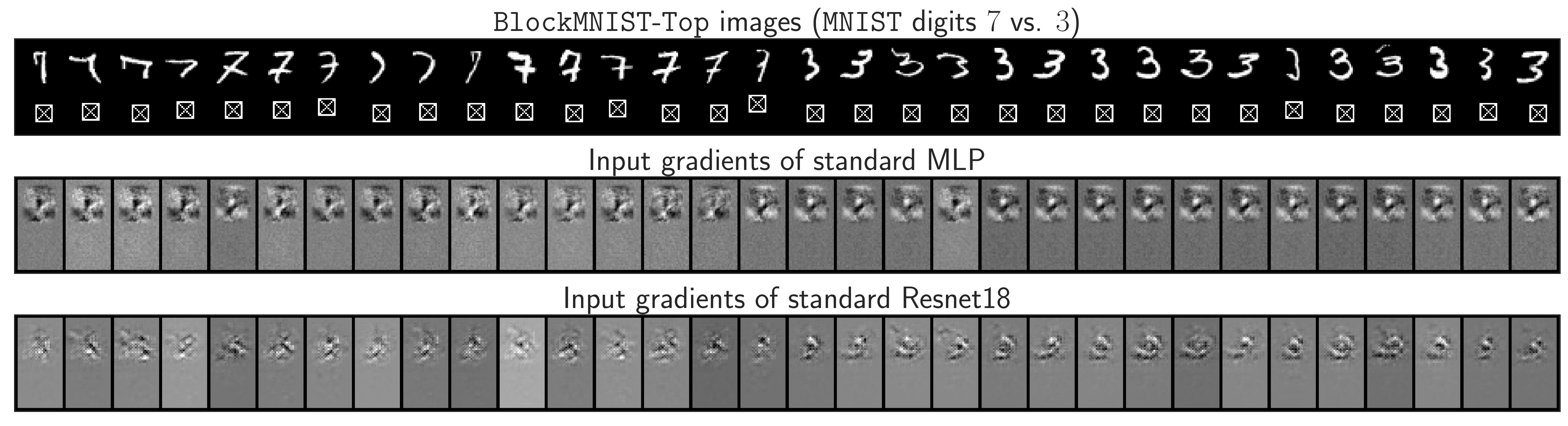}
	\vspace{-15px}
	\caption{
		\textbf{BlockMNIST-Top 3 vs. 7}.
		$20$ \semirealnameone~(\texttt{MNIST} $3$ vs. $7$) images and input gradients of standard MLP and Resnet18 models.
		As shown in the first row, 	the signal \& null blocks are fixed at the top \& bottom respectively in \semirealnameone~images.
		In contrast to results on \semirealname~in~\cref{fig:semireal_robust}, input gradients of standard models trained on~\semirealnameone~highlight the signal block, suppress the null block, and satisfy \premise. 
		Please see~\Cref{app:leakage} for details.
	} 
 \label{fig:app_blockmnist_top_73}
 \vspace{-5px}
\end{figure*}
		
\begin{figure*}[t]
	\centering
	\includegraphics[width=\linewidth]{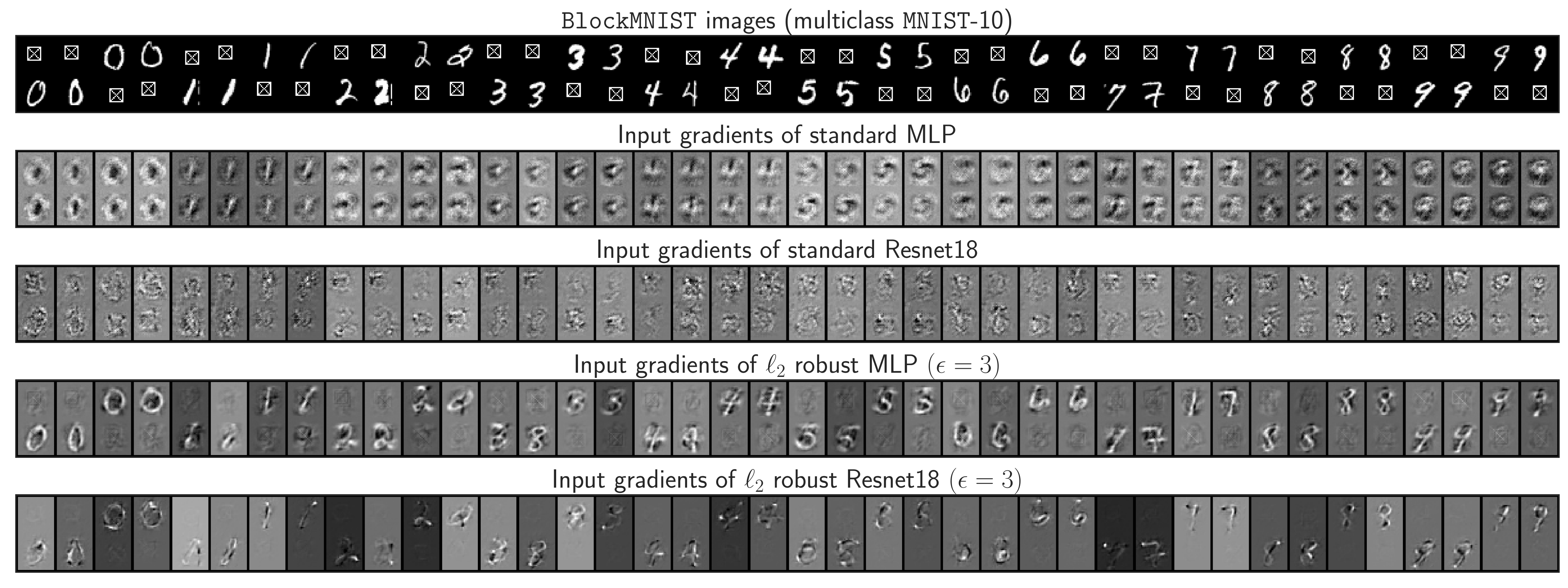}
	\vspace{-15px}
	\caption{
	\textbf{Multiclass BlockMNIST}.
	$40$ \semirealname~(all \texttt{MNIST} classes) images and their corresponding input gradients. 
	dataset. 
	In this setting, the \emph{signal} block, containing an \texttt{MNIST} digit sampled from a class chosen uniformly at random, determines the image class $y \in \{0,\ldots,9\}$. 
	The second, third, and fourth rows show input gradients of standard Resnet18, standard MLP, $\ell_2$ robust Resnet18  $(\epsilon=3)$ and $\ell_2$ robust MLP $(\epsilon=3)$ respectively.
	The plots clearly show that input gradients of standard \semirealname~models incorrectly highlight \emph{the non-discriminative null block} as well, thereby violating \premise.
	In contrast, input gradients of robust models highlight the signal block, suppress the null block, and satisfy \premise. 
	See~\Cref{app:leakage} for details.
	} 
 \label{fig:app_blockmnist_multiclass}
 \vspace{-5px}
\end{figure*}

\begin{figure*}[t]
	\centering
	\includegraphics[width=\linewidth]{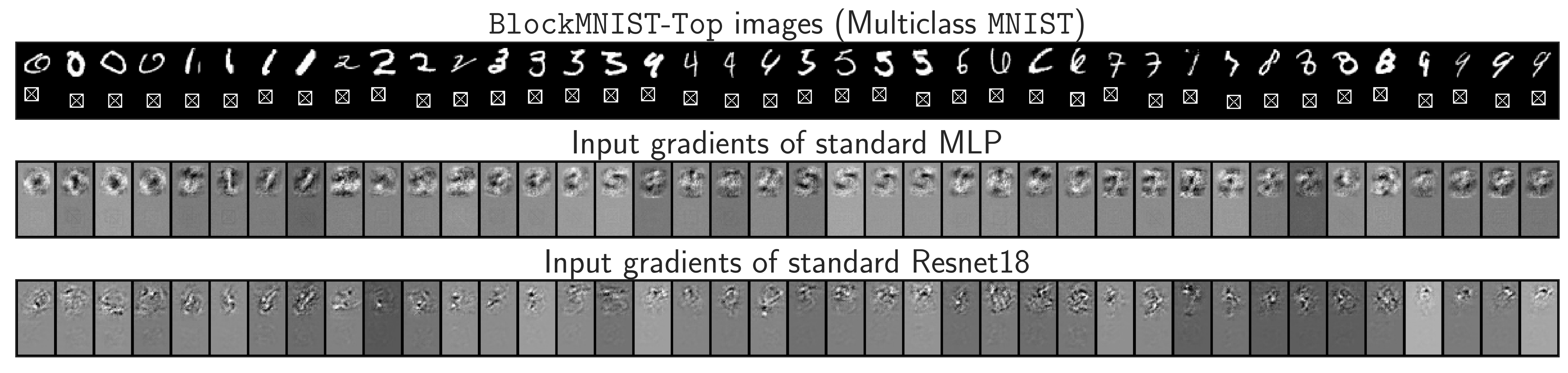}
	\vspace{-15px}
	\caption{
		\textbf{Multiclass BlockMNIST-Top}.
		$40$ \semirealnameone~(all \texttt{MNIST} classes) images and corresponding input gradients of standard MLP and Resnet18 models.
		As shown in the first row, 	the signal \& null blocks are fixed at the top \& bottom respectively in \semirealnameone~images.
		In contrast to results on \semirealname~in~\cref{fig:semireal_robust}, input gradients of standard models trained on~\semirealnameone~highlight the signal block, suppress the null block, and satisfy \premise. 
		Please see~\Cref{app:leakage} for details.
	} 
 \label{fig:app_blockmnist_top_multiclass}
 \vspace{-5px}
\end{figure*}

\subsection{Does randomness in initialization explain why input gradients violate \premise?}
In this section, \emph{we investigate whether the poor quality of input gradients in standard models is due to randomness retained from the initialization}. Figure~\ref{fig:app_init} shows scatter plots of input gradient values over all pixels in all images before (x-axis) and after (y-axis) standard training on four image classification benchmarks. The results indicate that (i) the scale of gradients after training is at least an order of magnitude larger than those before training and (ii) the gradient values before and after training are uncorrelated. Together, these results suggest that random initialization does not have much of a role in determining the input gradients after training.

\begin{figure*}[t]
	\centering
	\includegraphics[width=\linewidth]{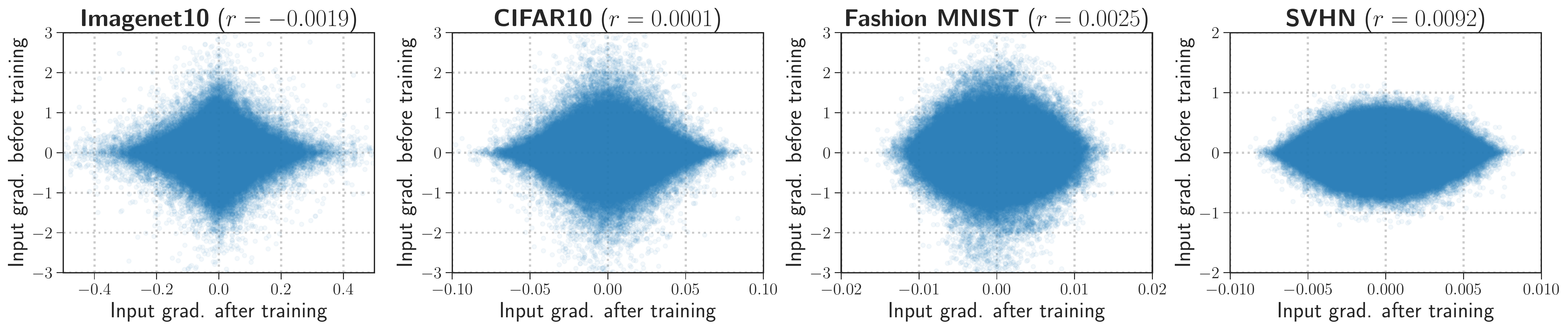}
	\vspace{-15px}
	\caption{\textbf{Does random initialization affect input gradients after training?} The scatter plots above show the gradient values at the beginning of training and after end of training on y-axis and x-axis respectively. We can see that the scale of gradients is much larger at the end of training compared to that at the beginning of training and both of them are uncorrelated. This suggests that the poor quality of input gradients of standard trained models is a result of the training process, and not because of random initialization.
	} 
 \label{fig:app_init}
 \vspace{-5px}
\end{figure*}

\subsection{Do other feature attribution methods exhibit feature leakage?}
\label{app:other_methods}
In this section, we evaluate feature leakage in five feature attribution methods: Integrated Gradients~\cite{sundararajan2016gradients}, Layer-wise Relevance Propagation (LRP)~\cite{10.1371/journal.pone.0130140}, Guided Backprop~\cite{springenberg2014striving}, Smoothgrad~\cite{smilkov2017smoothgrad} (with standard deviation $\sigma \in \{0.1, 0.3, 0.5\}$), and Occlusion~\cite{zeiler2013visualizing} (with patch size $\rho \in \{5, 10\}$).
First, we evaluate the aforementioned feature attribution methods on standard models trained on \semirealname~data.
As shown in~\Cref{fig:blockmnist_pval0.5_std-mlp} and~\Cref{fig:blockmnist_pval0.5_std-res18}, in addition to vanilla input gradients, all five feature attribution methods evaluated on standard MLPs and Resnet18 models highlight the \texttt{MNIST} signal block as well as the null block.
Conversely,~\Cref{fig:blockmnist_pval1.0_std-mlp} and~\Cref{fig:blockmnist_pval1.0_std-res18} show that when standard MLPs and Resnet18 models are trained on \semirealnameone~data, all feature attribution methods exclusively highlight the \texttt{MNIST} signal block.
These results collectively indicate that similar to vanilla input gradient attributions, multiple feature attribution methods exhibit feature leakage.
Furthermore, consistent with our findings on adversarial robustness vis-a-vis feature leakage,~\Cref{fig:blockmnist_pval0.5_rob-mlp} and~\Cref{fig:blockmnist_pval0.5_rob-res18} show that feature attribution method evaluated on adversarially robust MLPs and Resnet18 model do not exhibit feature leakage on~\semirealname~data.

\begin{figure*}
	\centering 
	\includegraphics[width=\linewidth]{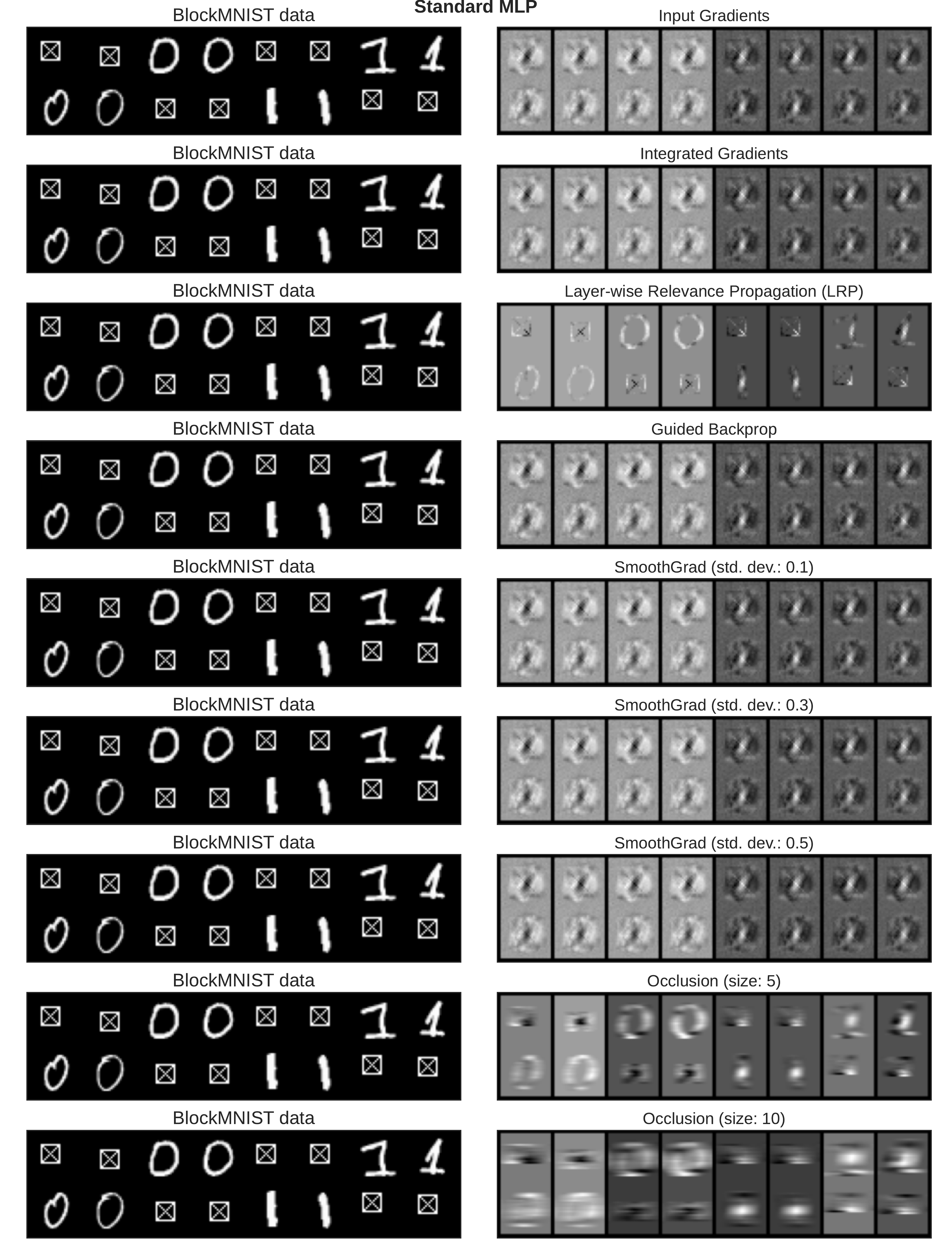}
	\vspace{-15px}
	\caption{Multiple instance-specific feature attribution methods evaluated using a standard two-layer MLP trained on~\semirealname~data. All feature attribution methods exhibit feature leakage, as the attributions highlight the non-predictive null block in addition to the \texttt{MNIST} signal block.}
	\label{fig:blockmnist_pval0.5_std-mlp}
\end{figure*}

\begin{figure*}
	\centering 
	\includegraphics[width=\linewidth]{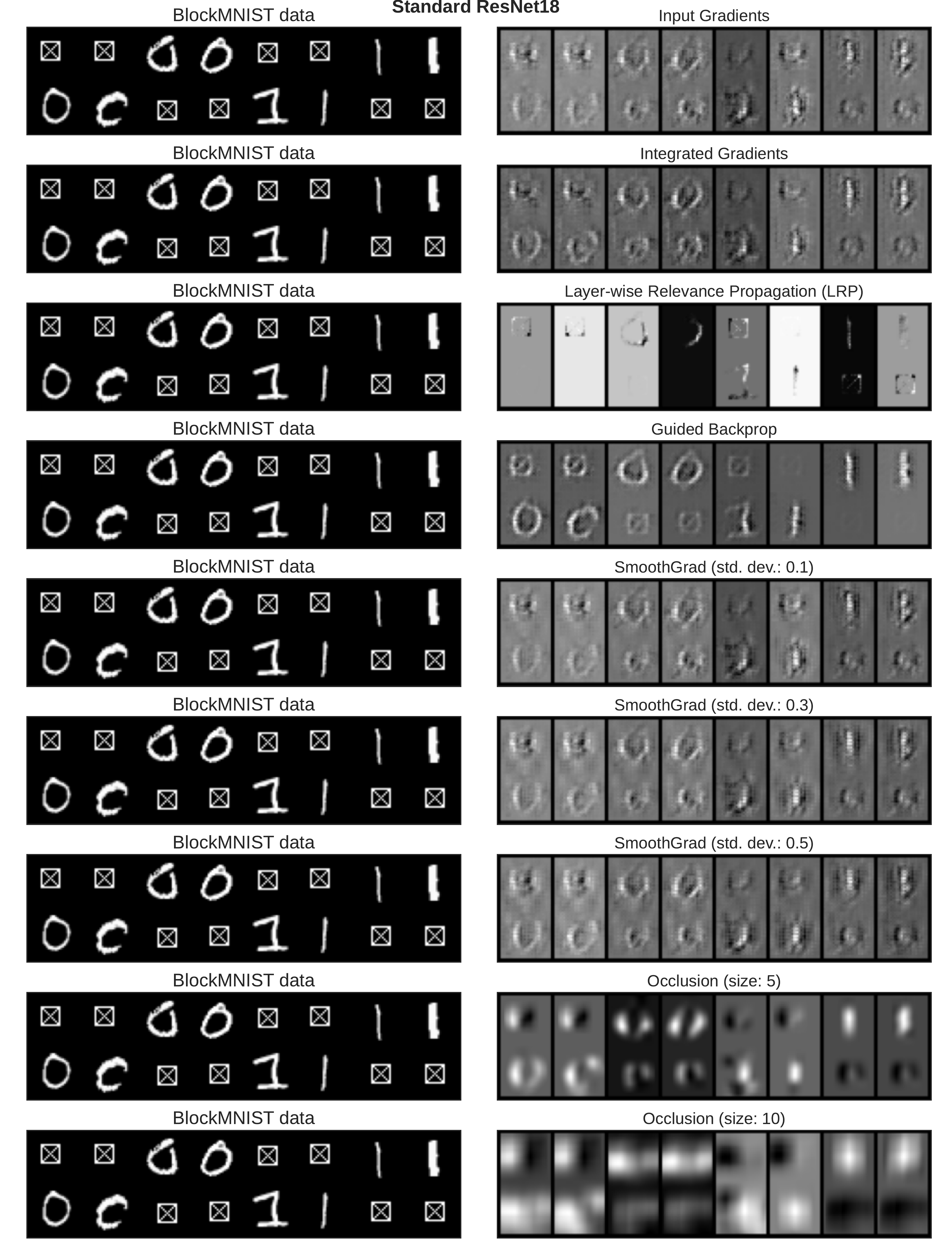}
	\vspace{-15px}
	\caption{Multiple instance-specific feature attribution methods evaluated using a standard ResNet18 trained on~\semirealname~data. All feature attribution methods exhibit feature leakage, as the attributions highlight the non-predictive null block in addition to the \texttt{MNIST} signal block. Surprisingly, in some cases, \texttt{LRP} (third row) exclusively highlights the null block.}
	\label{fig:blockmnist_pval0.5_std-res18}
\end{figure*}

\begin{figure*}
	\centering 
	\includegraphics[width=\linewidth]{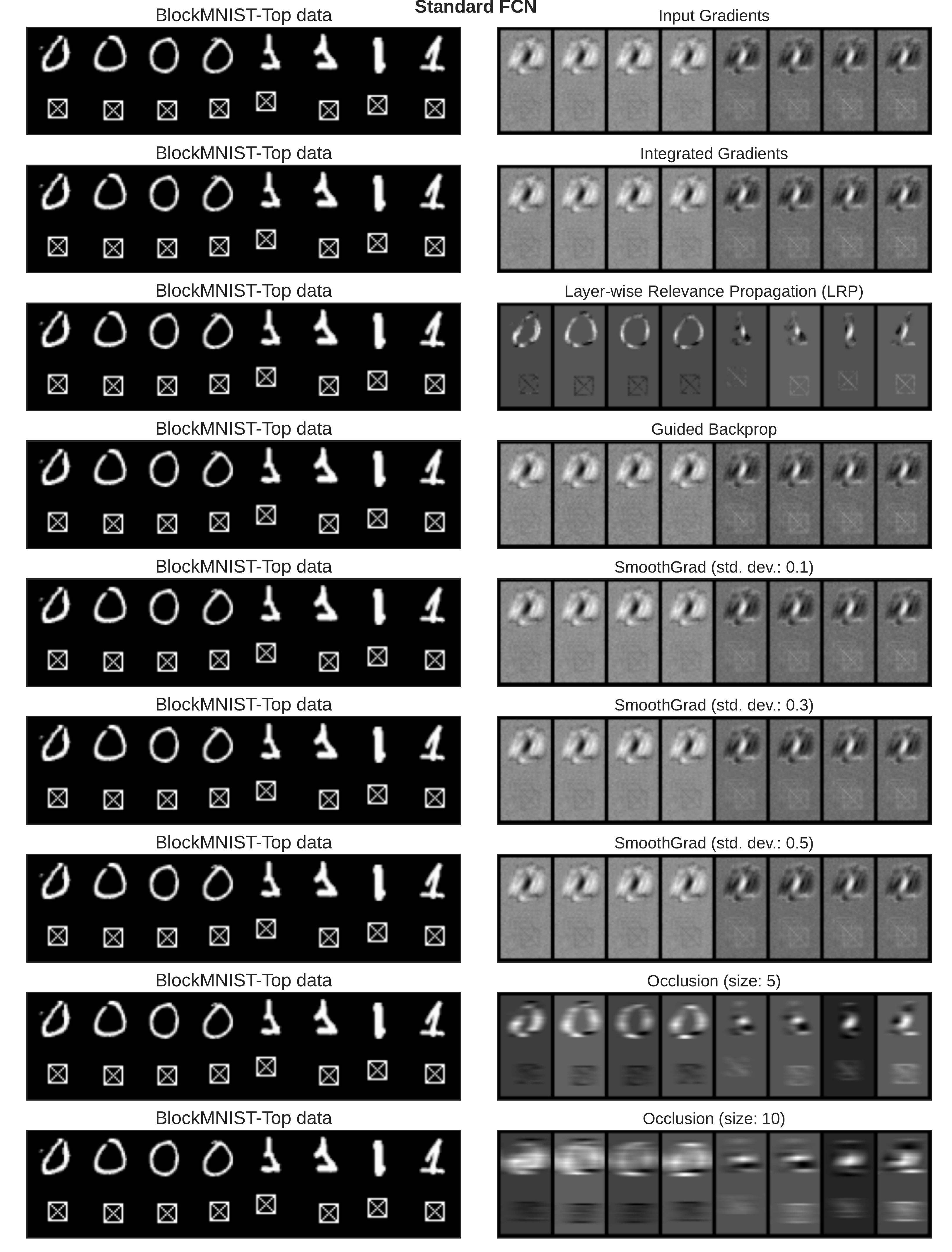}
	\vspace{-15px}
	\caption{Multiple instance-specific feature attribution methods evaluated using a standard two-layer MLP trained on~\semirealnameone~images, in which the \texttt{MNIST} signal block is fixed at the top. 
		On this dataset, feature attributions of all five methods highlight discriminative features in the signal block, suppress the null block, and satisfy \premise. 	
	}
	\label{fig:blockmnist_pval1.0_std-mlp}
\end{figure*}

\begin{figure*}
	\centering 
	\includegraphics[width=\linewidth]{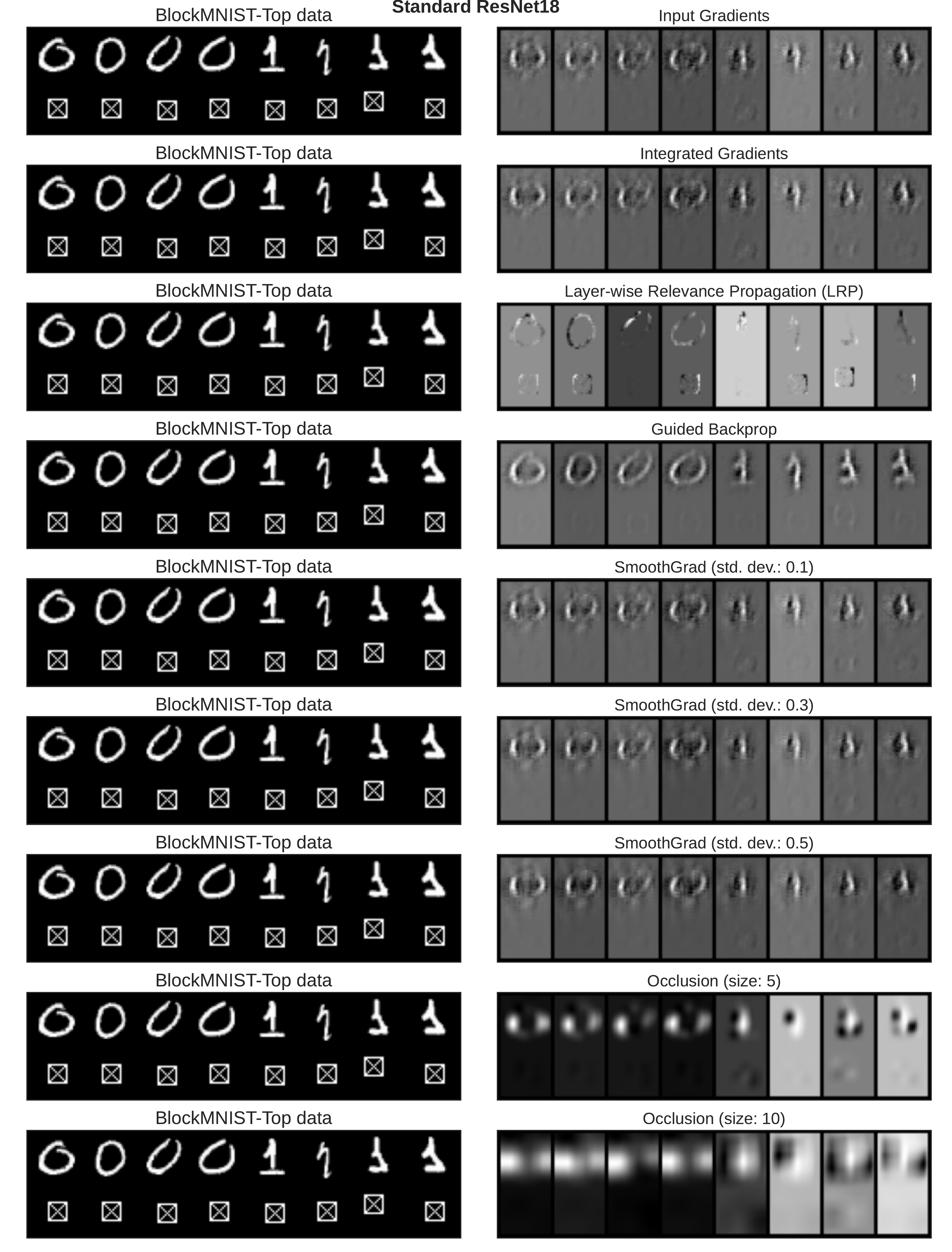}
	\vspace{-15px}
	\caption{Multiple instance-specific feature attribution methods evaluated using a standard ResNet18 trained on~\semirealnameone~data, in which the \texttt{MNIST} signal block is fixed at the top. 
		On this dataset, feature attributions of all five methods highlight discriminative features in the signal block, suppress the null block, and satisfy \premise. 	
		Surprisingly, \texttt{LRP} attributions (third row) highlight the null patch of \semirealnameone~images as well.}
	\label{fig:blockmnist_pval1.0_std-res18}
\end{figure*}

\begin{figure*}
	\centering 
	\includegraphics[width=\linewidth]{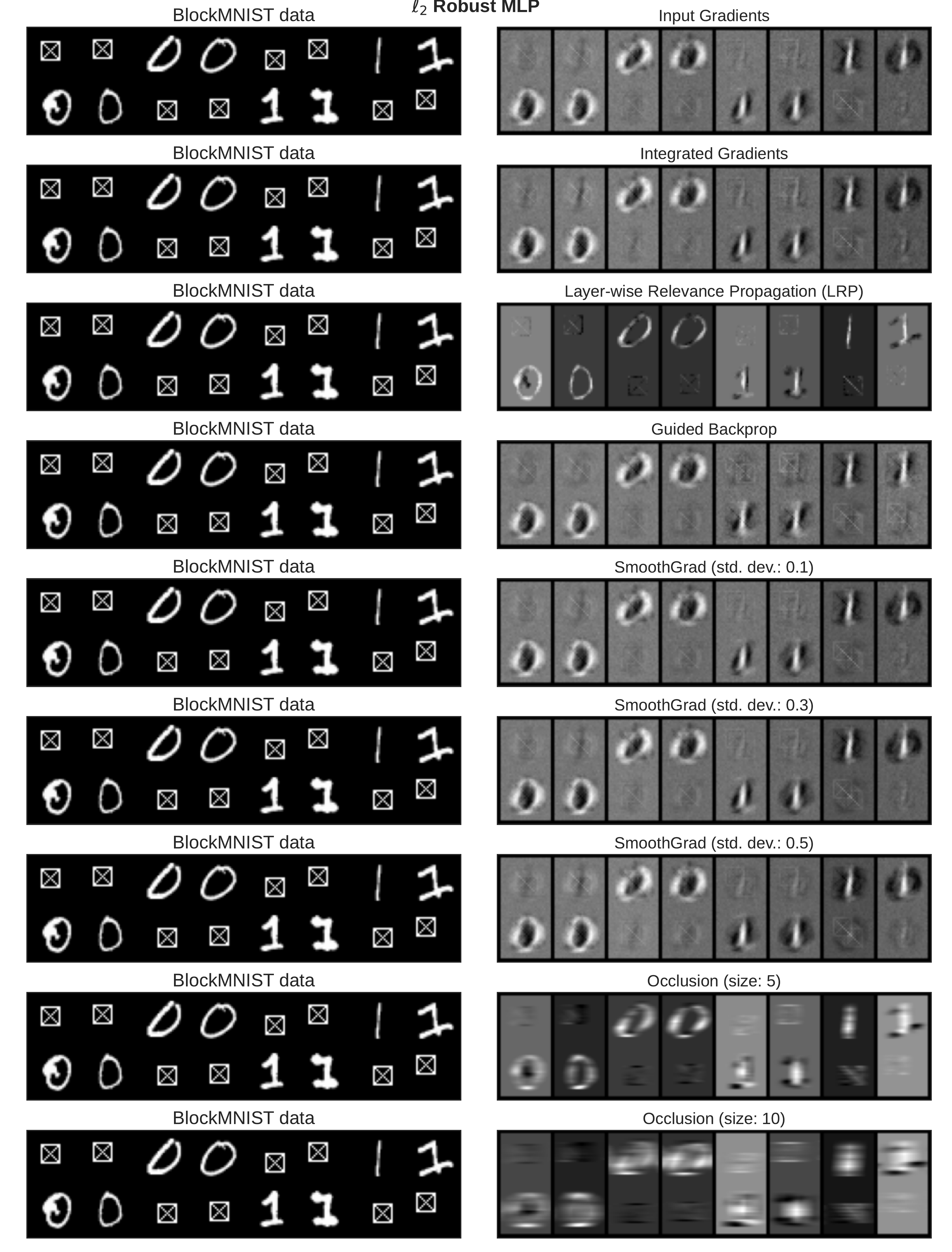}
	\vspace{-15px}
	\caption{Multiple instance-specific feature attribution methods evaluated using a $\ell_2$ robust two-layer MLP trained on~\semirealname~data. Consistent with our findings on adversarial robustness vis-a-vis feature leakage, feature attributions of all methods of robust models do not exhibit feature leakage.}
	\label{fig:blockmnist_pval0.5_rob-mlp}
\end{figure*}

\begin{figure*}
	\centering 
	\includegraphics[width=\linewidth]{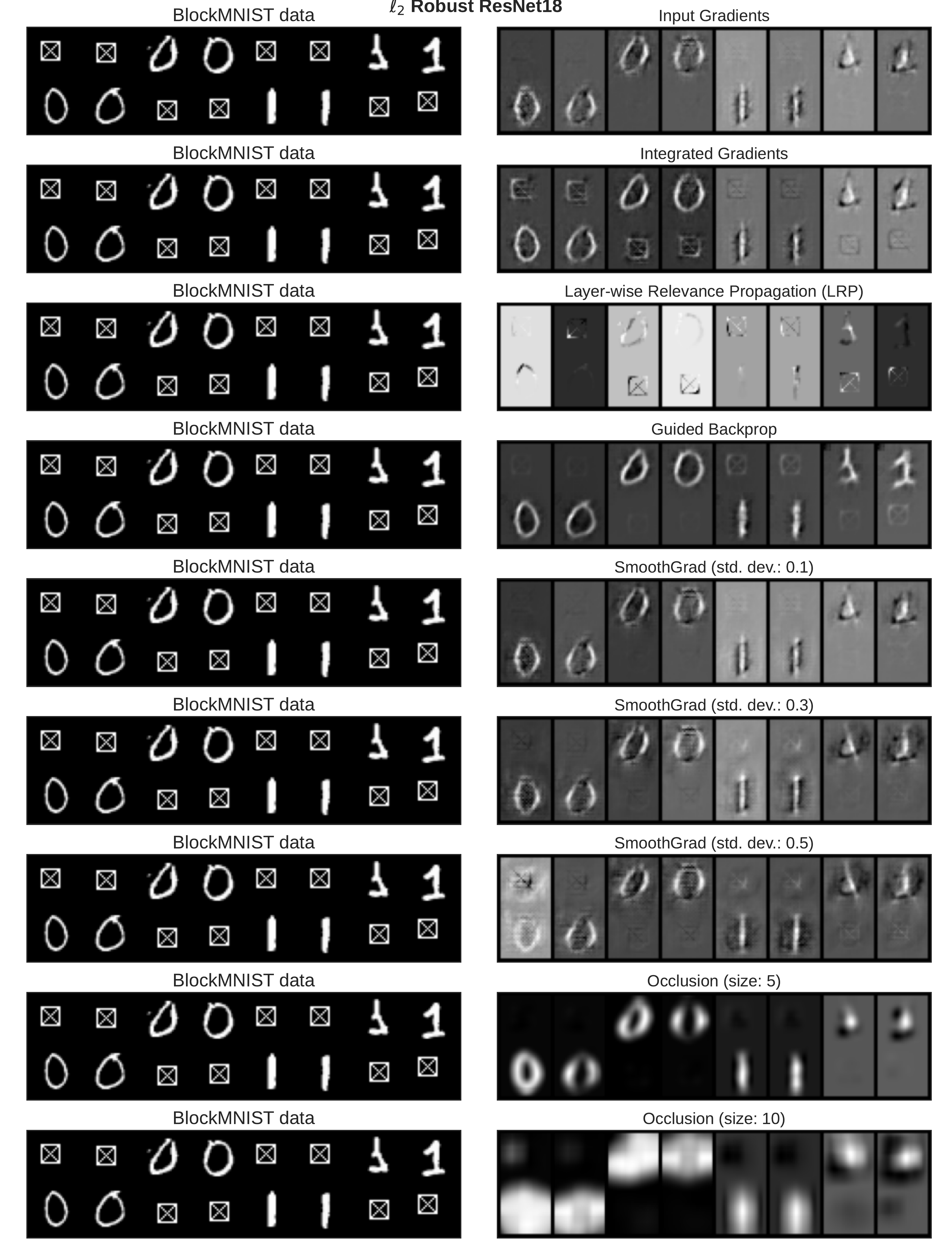}
	\vspace{-15px}
	\caption{Multiple instance-specific feature attribution methods evaluated using a $\ell_2$ robust ResNet18 trained on~\semirealname~data. Consistent with our findings on adversarial robustness vis-a-vis feature leakage, feature attributions of all methods (except \texttt{LRP}) of robust models do not exhibit feature leakage.}
	\label{fig:blockmnist_pval0.5_rob-res18}
\end{figure*}

%

\clearpage

\section{Proof of Theorem~\ref{thm:main}}\label{app:theory}
We first begin with the definition of a function, $\psi: \R^2 \rightarrow \R$ which will prove useful in the analysis:
\begin{align}
	\psi(a,b) &\defeq \phi(a+b) - \phi(-a+b) = \left\{
	\begin{array}{cc}
		a - b & \mbox{ if } a \leq -\abs{b} \\
		0 & \mbox{ if } b \leq 0, \; \abs{a} \leq \abs{b} \\
		2a & \mbox{ if } b \geq 0, \; \abs{a} \leq \abs{b} \\
		a + b & \mbox{ if } a \geq \abs{b} \\
	\end{array},\right.\label{eqn:psi}
\end{align}
where we recall that $\phi(a) = \max(a,0)$ is the ReLU nonlinearity.
\begin{proof}[Proof of Theorem~\ref{thm:main} in the rich regime]
	We first claim that the max-margin classifier~\eqref{eqn:F1-max} is given by $\nu^* = \frac{1}{2}\delta_{\theta_1^*} + \frac{1}{2}\delta_{\theta_2^*}$, where $\theta_1 \defeq \left(\frac{1}{\sqrt{2}}, \frac{1}{\sqrt{2((d/2)+(1-(\eps d/2)^2)}}z, \frac{(1-(\eps d/2))}{\sqrt{2((d/2)+(1-(\eps d/2))^2)}}\right)$ and $\theta_2 \defeq \left(\frac{-1}{\sqrt{2}}, \frac{-1}{\sqrt{2((d/2)+(1-(\eps d/2))^2)}}z, \frac{(1-(\eps d/2))}{\sqrt{2((d/2)+(1-(\eps d/2))^2)}}\right)$ with $z \in \R^{\dtilde \cdot d}$ denoting the concatenation of $d/2$ copies of $u^* \in \R^{\dtilde}$ and $d/2$ copies of $0$ vectors of dimension $\dtilde$ each. To do so, we use~\citep[Proposition~12]{chizat2020implicit} which requires us to verify that there exists a probability distribution $p^*$ over the training data points such that:
	\begin{align}
		\textrm{Support}(\nu^*) &\in \argmax_{(w,r,b) \in \Sddtilde} \Eps{y \cdot \left({w \phi\left(\iprod{r}{x}+b\right)}\right)}, \mbox{ and } \label{eqn:supp}\\	
		\textrm{Support}(p^*) &\in \argmin_{(x,y) \in \D} y \cdot \Enus{w \phi\left(\iprod{r}{x}+b\right)}, \label{eqn:supp-p}
	\end{align}
	where $\Sddtilde$ denotes the unit sphere in $\R^{d\dtilde+2}$. In order to verify this condition, we use $p^* = \frac{1}{d} \sum_{\stackrel{j \in [d/2]}{y \in \set{\pm 1}}} \delta_{(y\utilde_j,y)}$, where $\utilde_j \in \R^{\dtilde \cdot d}$ is defined as the concatenation of $d$ vectors, each in $\R^{\dtilde}$, with the $j^{\textrm{th}}$ one being $u^*$, the remaining $[d/2]\setminus\set{j}$ being $-\eps u^*$ and the last $[d/2]$ being all zero vectors.
	
	We first prove~\eqref{eqn:supp-p}. Consider the point $(\uhat,y=1)$ in the support of the training distribution with $\uhat = (u^*+\eps g_1, \eps g_2,\cdots,\eps g_{d/2}, 0, \cdots,0)$. We see that:
	\begin{align*}
		& y \cdot \Enus{w \phi\left(\iprod{r}{\uhat}+b\right)} \\
		&= \frac{1}{2} \cdot \frac{1}{\sqrt{2}}\left( \phi\left(\frac{\iprod{z}{\uhat}}{\sqrt{2\left((d/2)+(1-(\eps d/2))^2\right)}} + \frac{1-(\eps d/2)}{\sqrt{2\left((d/2)+(1-(\eps d/2))^2\right)}}\right) \right. \\ &\qquad \qquad \left. - \phi\left(\frac{-\iprod{z}{\uhat}}{\sqrt{2\left((d/2)+(1-(\eps d/2))^2\right)}} + \frac{1-(\eps d/2)}{\sqrt{2\left((d/2)+(1-(\eps d/2))^2\right)}}\right)\right)
	\end{align*}
	Since $\iprod{z}{\uhat} = 1 + \eps \sum_{i \in [d/2]} \iprod{g_i}{u^*} \geq 1 - (\eps d/2) > 0$. Consequently, using~\eqref{eqn:psi}, we have that:
	\begin{align*}
		y \cdot \Enus{w \phi\left(\iprod{r}{\uhat}+b\right)} &\geq \frac{1}{2 \sqrt{2}} \cdot \frac{2 \left(1-(\eps d/2)\right)}{\sqrt{2\left((d/2)+(1-(\eps d/2))^2\right)}} = y \cdot \Enus{w \phi\left(\iprod{r}{y\utilde_j}+b\right)}.
	\end{align*}
	This proves~\eqref{eqn:supp-p}.
	
	We now prove~\eqref{eqn:supp}.
	For $\theta = (w,r,b)$, denote $\L(\theta) \defeq \ED{y\cdot \left({w \phi\left(\iprod{r}{x}+b\right)}\right)}$.
	We have $\L(\theta_1) = \L(\theta_2) = \frac{1}{d} \cdot \sum_{j \in [d/2]}\frac{1}{\sqrt{2}}\frac{2(1-(\eps d/2))}{\sqrt{2((d/2)+(1-(\eps d/2))^2)}} = \frac{1-(\eps d/2)}{2\sqrt{(d/2)+(1-(\eps d/2))^2}}$. We will now show that $\max_{\theta \in \Sd} \L(\theta) = \frac{1-(\eps d/2)}{2 \sqrt{(d/2)+(1-(\eps d/2))^2}}$. For a given $\theta = (w,r,b)$, we first show that it is sufficient to consider the case where $r$ is the concatenation of $\alpha_i u^*$ for some $\alpha_1,\cdots,\alpha_d$.
	In order to do this, given any $\theta = (w,r,b)$, let $\alpha_i \defeq \iprod{r_i}{u^*}$ for $i \in [d/2]$ denote the inner product of the $i^{\textrm{th}}$-block vector of $r$ with $u^*$ and let $\alphabar \defeq \frac{1}{d/2} \sum_{i \in [d/2]} \alpha_i$ be the mean of $\alpha_i$. The function $\L(\theta)$ can be simplified as:
	\begin{align*}
		\L(\theta) = \frac{w}{d} \sum_{i \in [d/2]} \left(\phi\left(\alpha_i - (\eps d/2) \alphabar  + b\right) - \phi\left(-\alpha_i + (\eps d/2) \alphabar  + b\right)\right).
	\end{align*}
	We can now consider $r'$ to be the concatenation of $\iprod{r_i}{u^*} u^*$ for $i \in [d/2]$ and the remaining coordinates equal to zero, which ensures that $\L(\theta) = \L(\theta')$ for $\theta' = (w,r',b)$ and $\norm{\theta} \geq \norm{\theta'}$. We can then choose $\abs{w'} \geq \abs{w}$ such that $\norm{(w',r',b)}=1$ and $\L((w',r',b)) > \L(\theta)$. So it suffices to consider $\L(\theta)$ for $\theta = (w,r,b)$ where $r$ is the concatenation of $\alpha_i u^*$ for some $\alpha_i$ for $i \in [d/2]$ and the remaining coordinates being set to zero.
	Let us consider two situations separately:
	
	\textbf{Case I, $b \geq 0$}: 
	Recall from~\eqref{eqn:psi} the definition $\psi(a,b) \defeq \phi(a+b) - \phi(-a+b)$. First note from~\eqref{eqn:psi} that, $\abs{\psi(a,b) - \psi(a',b)} \leq 2\abs{a-a'}$ and $\psi(a,b) - \psi(a',b) \geq a-a'$ for every $a > a'$.
	If $\alpha_j < 0$ for some $j$, then choosing $r'_j = r_j - 2 \alpha_i w^*$ with $r'_i = r_i$ for all $i \neq j$ gives us a corresponding $\theta'$ satisfying
	\begin{align*}
		\L(\theta') &\geq \frac{w}{d} \sum_{i \neq j} \left(\phi\left(\alpha_i - (\eps d/2) \alphabar  + b\right) - \phi\left(-\alpha_i + (\eps d/2) \alphabar  + b\right)\right) - 2 \eps (d/2-1) \abs{\alpha_i} \\
		&\quad + \frac{w}{d}\left(\phi\left(\alpha_i - (\eps d/2) \alphabar  + b\right) - \phi\left(-\alpha_i + (\eps d/2) \alphabar  + b\right)\right) + (1-\eps) \abs{\alpha_i} \\
		&\geq \L(\theta) + \frac{w}{d} \cdot \left(1-\eps d\right) \cdot \abs{\alpha_i}.
	\end{align*}
	So, it suffices to restrict our attention to $\theta$ such that $\alpha_i \geq 0$ for all $i$ in order to prove~\eqref{eqn:supp}. We will now show that making all $\alpha_i$ equal will further increase the value of $\L$. In order to see this, let $\alpha_1 = \min_{i \in [d/2]} \alpha_i$ and $\alpha_2 = \max_{i \in [d/2]} \alpha_i$. Then constructing $r'$ from $r$ by replacing $\alpha_1$ and $\alpha_2$ with $\alpha' \defeq \frac{\alpha_1 + \alpha_2}{2}$ ensures that $\norm{r'} \leq \norm{r}$ while at the same time $\L((w,r',b)) - \L((w,r,b))$ since $\alpha_1 \geq 0$ implies $\alpha_1 - (\eps d/2) \alphabar > -\left(\alpha_2 - (\eps d/2) \alphabar\right)$. If $\alpha_i = \alpha_j$ for all $i,j \in [d]$, then from~\eqref{eqn:psi},
	\begin{align*}
		\L(\theta) &= \frac{w}{d} \sum_{i \in [d]} \alpha_i - (\eps d/2) \alphabar + \min\left(\alpha_i - (\eps d/2) \alphabar,b\right)
		= {(1- (\eps d/2) )w}{} \left(\alphabar + \min\left(\alphabar,\frac{b}{1- (\eps d/2) }\right)\right).
	\end{align*}
	The maximizer of the above expression under the constraint $\norm{\theta}^2 = w^2 + (d/2) \alphabar^2 + b^2 = 1$ can be seen to be when $w= \pm \frac{1}{\sqrt{2}}$, $\alphabar = \frac{2}{\sqrt{(d/2)+(1-(\eps d/2))^2}}$ and $b = \frac{1-(\eps d/2)}{\sqrt{2\left((d/2)+(1-(\eps d/2))^2\right)}}$ achieving value $\frac{1-(\eps d/2)}{2 \sqrt{(d/2)+(1-(\eps d/2))^2}}$.
	
	\textbf{Case II, $b < 0$}: In this case, we have from~\eqref{eqn:psi} that
	\begin{align*}
		\L(\theta) &\leq \frac{\abs{w}}{d} \sum_{i \in [d]} \alpha_i - (\eps d/2) \alphabar = {(1-(\eps d/2))\abs{w} \alphabar}{} \leq \frac{(1-(\eps d/2)) \abs{w} \norm{r}}{\sqrt{d/2}} \leq \frac{1-(\eps d/2)}{2 \sqrt{d/2}} \\ &< \frac{1-(\eps d/2)}{\sqrt{(d/2)+(1-(\eps d/2))^2}},
	\end{align*}
	where we used $\eps < \frac{1}{10 d}$ in the last step.
	This shows that $\nu^* = \frac{1}{2}\delta_{\theta_1^*} + \frac{1}{2}\delta_{\theta_2^*}$ is a max-margin classifier satisfying~\eqref{eqn:F1-max}.
	
	\textbf{Gradient magnitude}: For any input $(x,y)$, we note that the input gradient is of the form $\nabla_x \L(\nu^*,(x,y)) = \alpha z$ for some $\alpha \neq 0$. Consequently, the claim about the gradient magnitudes in different coordinates follows from the structure of $z$ proved above.

\end{proof}

\newpage 
\section{{Effect of adversarial training}}\label{app:conj}
\label{appendix-adv-train}

Consider training a model that is adversarially robust in an $\ell_p$ ball of radius $\epsilon$. Assuming that the inner iterations of adversarial training find the optimal perturbations, it can be shown that if adversarial training converges asymptotically {(i.e., in the rich regime)}, it does so to an appropriate max-margin classifier~\cite{chizat-comm}:
\begin{align}\label{eqn:adv-max-margin-orig}
	\nu^*\defeq \argmax_{{\nu \in \P\left(\Sddtilde\right)}} \min_{(x,y) \in B_{p}(\D,\epsilon)} y \cdot f(\nu,x),
\end{align}
where $B_p(\D,\epsilon) \defeq \set{(x,y):(\tilde{x},y) \sim \D, \norm{x-\tilde{x}}_p \leq \epsilon}$. This implies that using the techniques of previous section, we should be able to analyze the input gradient. However, such analysis requires an explicit form of the max-margin classifier defined above. In contrast to the standard training studied above, computing explicit form of the max-margin classifier is significantly non-trivial in the adversarially training case even for the simple special case of $\dtilde=1, \eta=0$ and $u^*=1$. {While we are unable to explicitly compute the max-margin classifier even for this case}, we make the following conjecture about the max-margin classifier. 
\begin{conjecture}\label{conj}
	Let data distribution $\D$ follow~\eqref{eqn:synth} with $\dtilde=1, \eta=0$ and $u^*=1$. Then, the classifier $\nut$ defined below is a max-margin classifier for adversarial training \eqref{eqn:adv-max-margin-orig} for $p=\infty$ and $\epsilon$ close to $0.5$:
	\begin{align}\label{eqn:adv-max-margin}
		{\nut} \defeq \frac{1}{d} \sum_{i \in [d/2]} \delta_{\theta_i} + \delta_{\theta'_i},
	\end{align}
	with $\theta_i \defeq (\frac{1}{\sqrt{2}}, \frac{3}{\sqrt{20}}e_i, \frac{-1}{\sqrt{20}}), \theta'_i \defeq (\frac{-1}{\sqrt{2}}, \frac{-3}{\sqrt{20}}e_i, \frac{-1}{\sqrt{20}})$ where $e_i \in \R^{d}$ denotes $i^\textrm{th}$ standard basis vector.
\end{conjecture}
Figure~\ref{fig:synth-adv} empirically verifies two consequences of this conjecture.
In Figure~\ref{fig:synth-adv}(a), we show that first-layer weights with large alignment with standard basis vectors also have large second-layer weights, indicating that axis-aligned first-layer weights are highly influential in the final model's prediction.
Figure~\ref{fig:synth-adv}(b) shows that the biases in first-layer ReLU units are predominantly negative.

\begin{figure}[h]
	\centering
	\includegraphics[width=0.58\linewidth]{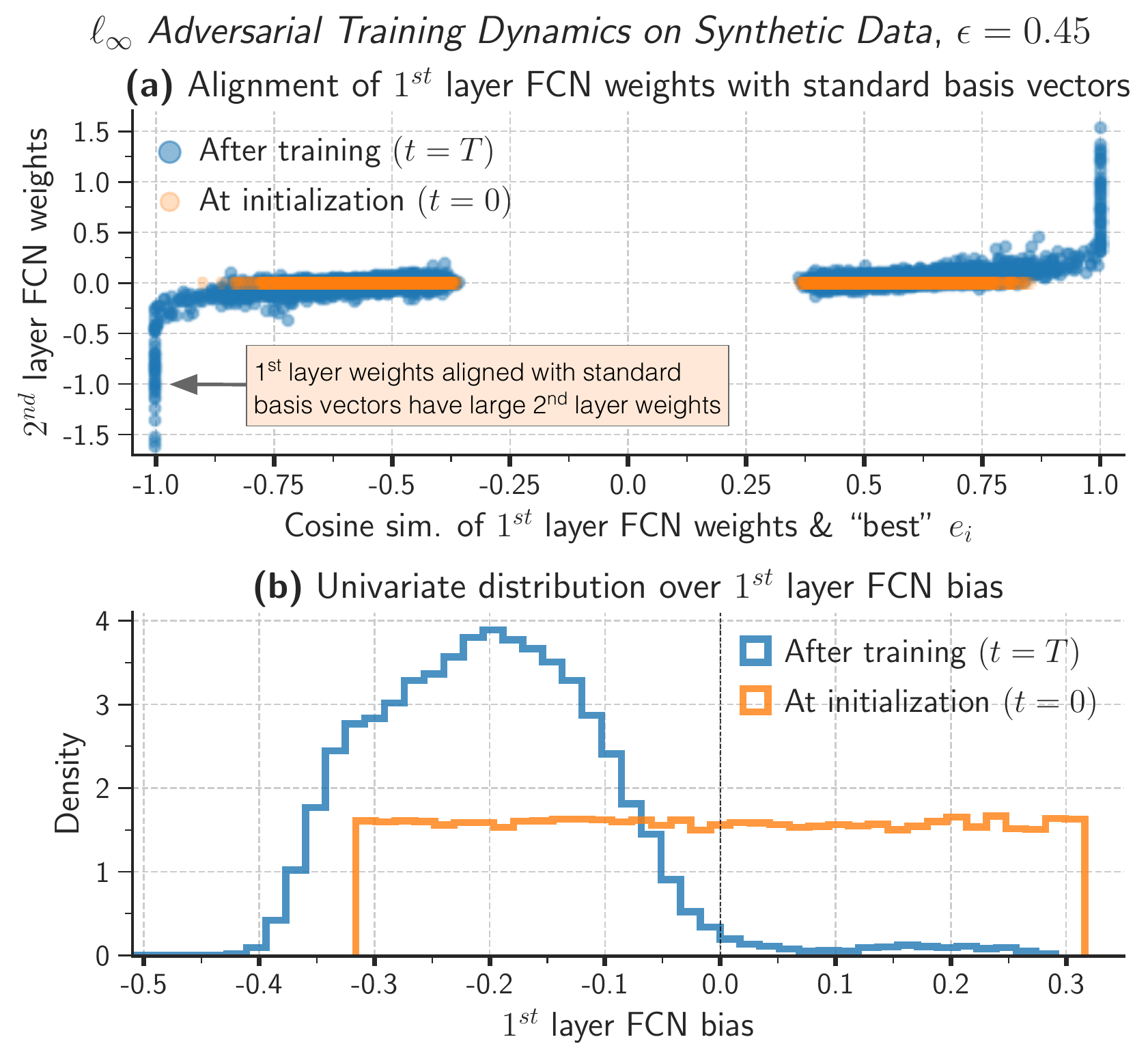}
	\vspace{-5px}
	\caption{
		Adversarial training dynamics for training one-hidden-layer FCNs with width $50,000$ on $10$-dimensional synthetic data.
		Subplot (a) shows that first-layer neurons aligned with standard basis vectors have large second-layer weights. 
		Given a normalized $1^{st}$ layer weight vector (i.e., rescaled so that it has unit $\ell_2$ norm), the x-axis in (a) plots the coordinate with largest magnitude in this normalized vector. Note that the gap around origin is due to that fact that the largest magnitude coordinate in a unit $\ell_2$ norm vector in $d=10$ dimensions is at least $\frac{1}{\sqrt{d}}\approx 0.32$.
		Subplot (b) shows that $\ell_{\infty}$ adversarial training results in first-layer bias terms that are predominantly negative.
		These observations support Conjecture~\ref{conj} and indicate that adversarial training quickly enters the rich regime.
	}
	\label{fig:synth-adv}
\end{figure}
\clearpage
%
%
The following lemma shows that the input gradients of $\nut$ in~\eqref{eqn:adv-max-margin} indeed highlight $j^*(x)$.
\begin{lemma}\label{lem:adv}
	Let data distribution $\D$ follow~\eqref{eqn:synth} with $\dtilde=1, \eta=0$ and $u^*=1$ and let $\nut$ be as defined in~\eqref{eqn:adv-max-margin}.
	Then, for any $(x,y) \sim \D$, we have:
	${\nabla_x \L(\nut,(x,y))} = c \cdot e_{j^*(x)},$ where $c \neq 0$ is a constant.
\end{lemma}
Assuming Conjecture~\ref{conj}, this lemma shows that {for the special case $\dtilde=1$ and $\eta=0$}, adversarially trained models have input gradients that reveal instance-specific features important for classification. Conjecture~\ref{conj} and Lemma~\ref{lem:adv} also explain several other empirically observed properties of adversarial training such as visually perceptible input gradients and adversarial examples~\cite{santurkar2019image}.
In this section, we prove Lemma~\ref{lem:adv}.
\begin{proof}[Proof of Lemma~\ref{lem:adv}]
	Given the classifier $\nut$ and a data point $(x,y)$, the input gradient is given by
	\begin{align*}
		\nabla_x \L(\nut,(x,y)) &= c \cdot \nabla_x f(\nut,x) \\
		&= c \cdot \Enut{w \phi'\left(\iprod{r}{x}+b\right) r},
	\end{align*}
	where $c = \frac{-y\exp\left(-y \cdot f(\nut,x)\right)}{1+\exp\left(-y \cdot f(\nut,x)\right)}$. Recall from~\eqref{eqn:adv-max-margin} that
	\begin{align*}
		{\nut} = \frac{1}{d} \sum_{i \in [d/2]} \delta_{\theta_i} + \delta_{\theta'_i},
	\end{align*}
	where $\theta_i \defeq (\frac{1}{\sqrt{2}}, \frac{3}{\sqrt{20}}e_i, \frac{-1}{\sqrt{20}})$ and $\theta'_i \defeq (\frac{-1}{\sqrt{2}}, \frac{-3}{\sqrt{20}}e_i, \frac{-1}{\sqrt{20}})$, $e_i$ denotes the $i^\textrm{th}$ standard basis vector in $\R^d$. If $(x,y) = (ye_{j^*(x)},y)$ and $(w,r,b)\sim \delta_{\theta_i}$ or $(w,r,b)\sim \delta_{\theta'_i}$, then $\Pr\left[\phi'\left(\iprod{r}{x}+b
	\right) \neq 0\right]>0$ if and only if $i = j^*(x)$. Consequently, the only contribution the input gradient comes from $\delta_{\theta_{j^*(x)}}$ and $\delta_{\theta'_{j^*(x)}}$. So,
	\begin{align*}
		\nabla_x \L(\nut,(x,y))
		&= c \cdot \Enut{w \phi'\left(\iprod{r}{x}+b\right) r} \\
		&= c' \cdot \left(\mathbb{E}_{(w,r,b)\sim \delta_{\theta_{j^*(x)}}}\left[{w \phi'\left(\iprod{r}{x}+b\right) r}\right]\right.
		+ \left. \mathbb{E}_{(w,r,b)\sim \delta_{\theta'_{j^*(x)}}}\left[{w \phi'\left(\iprod{r}{x}+b\right) r}\right]\right) \\
		&= c'' \cdot e_{j^*(x)}.
	\end{align*}
	This proves the result.
\end{proof}

\end{document}